\documentclass{article}

\PassOptionsToPackage{numbers, compress}{natbib}
\usepackage{titletoc}

\usepackage[preprint]{neurips_2022}
\usepackage[utf8]{inputenc} % allow utf-8 input
\usepackage[T1]{fontenc}    % use 8-bit T1 fonts
\usepackage{hyperref}       % hyperlinks
\usepackage{url}            % simple URL typesetting
\usepackage{booktabs}       % professional-quality tables
\usepackage{amsfonts}       % blackboard math symbols
\usepackage{nicefrac}       % compact symbols for 1/2, etc.
\usepackage{microtype}      % microtypography
\usepackage{ulem}
\usepackage[table]{xcolor}
\usepackage{colortbl}
\usepackage{amsmath}
\usepackage{amsmath}
\usepackage{amssymb}
\usepackage{mathtools}
\usepackage{amsthm}
\usepackage{amsfonts}
\usepackage{fdsymbol}
\usepackage{amsfonts}
\usepackage{tikz}
\usepackage{algorithm}
\usepackage{algorithmic}
\usepackage{mathtools}  
\usepackage{graphicx}
\usepackage{adjustbox}
\usepackage{multicol}

\usepackage{wrapfig}

\usepackage{tabularx}
\usepackage{cuted}%%\stripsep-3pt

\usepackage{ragged2e}

\usepackage{subcaption}

\usepackage{enumitem}

\newtheorem{theorem}{Theorem}[section]
\newtheorem{proposition}[theorem]{Proposition}
\newtheorem{lemma}[theorem]{Lemma}

\newtheorem{definition}[theorem]{Definition}
\newtheorem{assumption}[theorem]{Assumption}
\newtheorem{remark}[theorem]{Remark}
\newtheorem{example}[theorem]{Example}

\newcommand{\Alcomment}[1]{{\color{cyan} #1}}

\title{Learning Discrete Latent Variable Structures with Tensor Rank Conditions}

\author{%
  Zhengming Chen$^{1,2}$, Ruichu Cai$^{1,*}$, Feng Xie$^{3}$, Jie Qiao$^{1}$,  \\
  \textbf{Anpeng Wu$^{2,4}$, Zijian Li$^{2}$, Zhifeng Hao$^{1}$, Kun Zhang$^{2,5,}$}\thanks{Corresponding author} \\
  1. School of Computer Science, Guangdong University of Technology\\
  2. Machine Learning Department, Mohamed bin Zayed University of Artificial Intelligence \\
  3. Department of Applied Statistics, Beijing Technology and Business University \\
  4. Department of Computer Science and Technology, Zhejiang University\\
  5. Department of Philosophy, Carnegie Mellon University\\
}

\begin{document}

\maketitle

\begin{abstract}

Unobserved discrete data are ubiquitous in many scientific disciplines, and how to learn the causal structure of these latent variables is crucial for uncovering data patterns. Most studies focus on the linear latent variable model or impose strict constraints on latent structures, which fail to address cases in discrete data involving non-linear relationships or complex latent structures. To achieve this, we explore a tensor rank condition on contingency tables for an observed variable set $\mathbf{X}_p$, showing that the rank is determined by the minimum support of a specific conditional set (not necessary in $\mathbf{X}_p$) that d-separates all variables in $\mathbf{X}_p$. By this, one can locate the latent variable through probing the rank on different observed variables set, and further identify the latent causal structure under some structure assumptions. We present the corresponding identification algorithm and conduct simulated experiments to verify the effectiveness of our method. In general, our results elegantly extend the identification boundary for causal discovery with discrete latent variables and expand the application scope of causal discovery with latent variables.

\end{abstract}

\section{Introduction}
Social scientists, psychologists, and researchers from various disciplines are often interested in understanding causal relationships between the latent variables that cannot be measured directly, such as depression, coping, and stress \citep{Silva-linearlvModel}. A common approach to grasp these latent concepts is to construct a measurement model. For instance, experts design a set of measurable items or survey questions that serve as indicators of the latent variable and then use them to infer causal relationships among latent variables \citep{bollen2002latent,bartholomew2011latent,cui2018learning}. 

Numerous approaches exist for addressing structure learning among latent variables. In particular, if the data generation process is assumed to be a linear relationship, known as \textit{linear latent variable models}, several approaches have been developed. These include the second-order statistic-based approaches \citep{Silva-linearlvModel,Kummerfeld2016,chen2024testing,Sullivant-T-separation}, high-order moments-based ones \citep{xie2020GIN,chen2022identification,cai2019triad, Adam21_identifiability}, matrix decomposition-based methods \citep{anandkumar2013learning,anandkumar2014tensor,anandkumar2015learning}, and copula model-based approaches \citep{cui2018learning}. Moreover, the hierarchical latent variable structure has been well-studied within the linear setting \citep{huang2022latent,xie2022identification,chen2023some,jin2023structural}. However, the linear assumption is rather restrictive and the discrete data in the real world could be more frequently encountered (e.g., responses from psychological and educational assessments or social science surveys \citep{eysenck2021junior,skinner2019analysis}), which does not satisfy the linear assumption.

\begin{figure*}[ht]
    \centering
    \begin{subfigure}[t]{0.5\textwidth}
        \centering
        \includegraphics[height=2.0in]{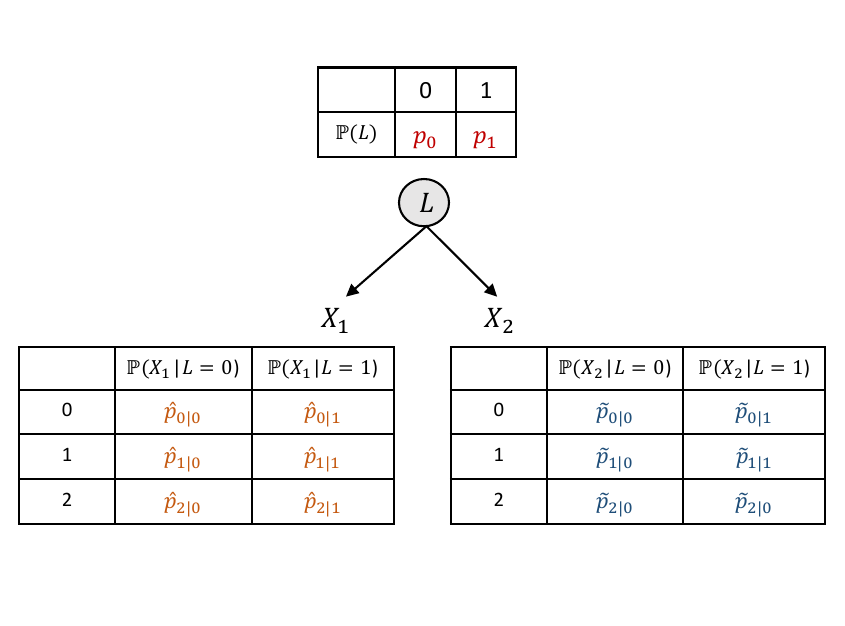}
        \caption{A latent structure with conditional probability tables.}
    \end{subfigure}%
    ~ 
    \begin{subfigure}[t]{0.5\textwidth}
        \centering
        \includegraphics[height=2in]{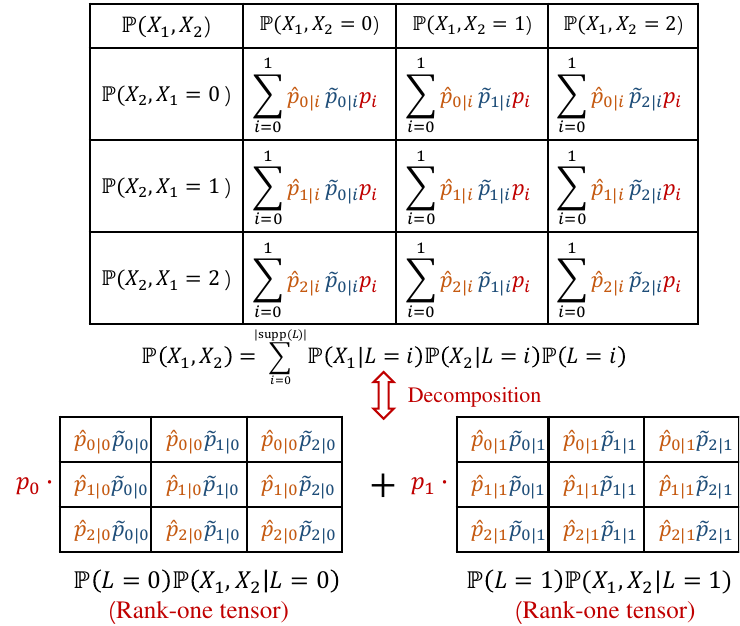}
        \caption{The decomposition of the joint distribution.}
    \end{subfigure}
    \caption{Illustrating for the graphical criteria of tensor rank condition such that a rank of the joint distribution is determined by the support of a specific conditional set that d-separates all observed variables, i.e., $\mathrm{Rank}(\mathbb{P}(X_1, X_2)) = |\mathrm{supp}(L)| = 2$.}
    \label{fig:exp}
\end{figure*}

% the challenging nonlinear transition relationship in discrete data, few identifiability results exist and are mostly only applicable in strict cases.

When the data generation process is discrete, however, due to the challenging nonlinear transition relationship in discrete data, few identifiability results exist and are mostly only applicable in strict cases. In particular, under some prespecified structure, the identifiability of parameters is established, such as in the hidden Markov model(HMM) \citep{anandkumar2012method} model, topic models \citep{anandkumar2014tensor}, and multiview mixtures model \citep{anandkumar2015learning}. By further specifying the latent variable structure as a tree, \cite{wang2017spectral,song2013hierarchical} show that the structural model is identifiable. Recently, \cite{gu2022blessing,gu2023bayesian} further considered the identifiability of pyramid structure under the condition that each latent variable has at least three observed children. However, challenges persist in extending identifiability to more general structures among discrete latent variables. Existing approaches, unfortunately, cannot identify the causal structure of latent variables as shown in Fig. \ref{fig:illust_simple_example}(a).

% will give incorrect or uninformative answers when the causal structure is shown in

In this paper, we seek to find out a general identification criteria to identify the discrete latent structure in the case where the structure is not limited to a tree-structured graph or pyramid structure. To achieve this, 
we explore a tensor rank condition on the contingency tables for an observed variable set $\mathbf{X}_p$, to probe the latent causal structure from observed data. Interestingly, as shown in Fig. \ref{fig:exp}, we found that the rank of the contingency tables of the joint distribution $\mathbb{P}(X_1,X_2)$ is deeply connected to the support of a variable $L$ (not necessary among $X_1, X_2$) that d-separate $X_1$ and $X_2$. By this observation, we first develop a general tensor rank condition for the discrete causal model and show that such a rank is determined by the minimal support of a specific conditional set (not necessary in $\mathbf{X}_p$) that d-separates all variables in $\mathbf{X}_p$.
Such findings intrigue the possibility to identify the discrete latent variables structure. We further propose a discrete latent structure model that accommodates more general latent structures and shows that the discrete latent variable structure can be identified locally and iteratively through tensor rank conditions. Subsequently, we present an identification algorithm to complete the identifiability of discrete latent structure models, including the measurement model and the structure model. We theoretically show that under proper causal assumptions, such as faithfulness and the Markov assumption, the measurement model is fully identifiable and the structure model can be identified up to a Markov equivalence class.

The contributions of this work are three-fold. (1) We first establish a connection between the
tensor rank condition and the graphical patterns in a general discrete causal model, including specific d-separation relations. (2) We then exploit the tensor rank condition to learn the discrete latent variable model, allowing flexible relations between latent variables. (3) We present a structure learning algorithm using tensor rank conditions and demonstrate the effectiveness of the proposed algorithm through simulation studies.

\begin{figure}[t]
	\begin{center}
		\begin{tikzpicture}[scale=.8, line width=0.5pt, inner sep=0.2mm, shorten >=.1pt, shorten <=.1pt]
		\draw (1.75, 0.8) node(L1) [circle, fill=gray!50,minimum size=0.3cm, draw] {{\footnotesize\,$L_{1}$\,}};

		\draw (0.5, 0) node(L2) [circle, fill=gray!50,minimum size=0.3cm, draw] {{\footnotesize\,$L_2$\,}};
		
		\draw (3, 0) node(L3) [circle, fill=gray!50,minimum size=0.3cm, draw] {{\footnotesize\,$L_3$\,}};

		\draw (1.75, -0.8) node(L4) [circle, fill=gray!50,minimum size=0.3cm, draw] {{\footnotesize\,$L_{4}$\,}};

		%for L1

		\draw (0.25, -0.8) node(X3) [] {{\footnotesize\,$X_4$\,}};
		\draw (1.0, -0.8) node(X4) [] {{\footnotesize\,$X_5$\,}};
        \draw (-.3, -0.8) node(X12) [] {{\footnotesize\,$X_6$\,}};

		%for L2
		\draw (1.25, 1.6) node(X1) [] {{\footnotesize\,$X_1$\,}};

		\draw (2.25, 1.6) node(X2) [] {{\footnotesize\,$X_2$\,}};
		
		\draw (0.45, 1.5) node(X9) [] {{\footnotesize\,$X_3$\,}};

		\draw (2.75, -0.8) node(X5) [] {{\footnotesize\,${X}_{7}$\,}};
		
		\draw (3.3, -0.8) node(X6) [] {{\footnotesize\,${X}_{8}$\,}};

        \draw (3.8, -0.8) node(X11) [] {{\footnotesize\,$X_{9}$\,}};

		%for L3
		\draw (1.0, -1.6) node(X7) [] {{\footnotesize\,$X_{10}$\,}};

		\draw (1.75, -1.6) node(X8) [] {{\footnotesize\,$X_{11}$\,}};
		
		\draw (2.45, -1.6) node(X10) [] {{\footnotesize\,$X_{12}$\,}};

		%for L4

		%
		\draw[color=blue!80,->] (L1) -- (L2) node[pos=0.5,sloped,above] {};
		\draw[color=blue!80,->] (L1) -- (L3) node[pos=0.5,sloped,above] {};
% 		\draw[->] (L1) -- (L4) node[pos=0.5,sloped,above] {};
        %\draw[color=blue,->] (L2) -- (L3) node[pos=0.5,sloped,above] {}; 
		\draw[color=blue!80,->] (L2) -- (L4) node[pos=0.5,sloped,above] {}; 
		\draw[color=blue!80,->] (L3) -- (L4) node[pos=0.5,sloped,above] {};

		\draw[color=red!80,->] (L1) -- (X1) node[pos=0.5,sloped,above] {};
		\draw[color=red!80,->] (L1) -- (X2) node[pos=0.5,sloped,above] {};
		
		\draw[color=red!80,->] (L1) -- (X9) node[pos=0.5,sloped,above] {};

		\draw[color=red!80,->] (L2) -- (X12) node[pos=0.5,sloped,above] {};
		
		\draw[color=red!80,->] (L2) -- (X3) node[pos=0.5,sloped,above] {};
		\draw[color=red!80,->] (L2) -- (X4) node[pos=0.5,sloped,above] {};

		\draw[color=red!80,->] (L3) -- (X5) node[pos=0.5,sloped,above] {};
		\draw[color=red!80,->] (L3) -- (X6) node[pos=0.5,sloped,above] {};
        \draw[color=red!80,->] (L3) -- (X11) node[pos=0.5,sloped,above] {};

		\draw[color=red!80,->] (L4) -- (X7) node[pos=0.5,sloped,above]{};
		\draw[color=red!80,->] (L4) -- (X8) node[pos=0.5,sloped,above] {};
		
		\draw[color=red!80,->] (L4) -- (X10) node[pos=0.5,sloped,above] {};

  \draw (1.5, -2.5) node(ii) [] {{\footnotesize\,(a) ground truth\,}};
		%\draw[color=red,->] (L4) -- (X11) node[pos=0.5,sloped,above] {};
		\end{tikzpicture}
        \begin{tikzpicture}[scale=.8,sibling distance=11em, level distance=1.3cm,
  every node/.style = {shape=rectangle, rounded corners,
    draw, align=center, top color=white, bottom color=white!20}]]
  \node  {\footnotesize Identifiability of Discrete LSM}
    child { node  {\footnotesize Identifiability  of \\measurement model}
        child{node {\footnotesize Completely identifiable}} 
    }
    child { 
    node {\footnotesize Identifiability of \\ structure model} 
    child{node {\footnotesize Markov equivalence}} 
    };
\node [color = white ,minimum width=0.5cm,minimum height=0.5cm,text=black](anchor)at(0,-3.8){\footnotesize\ (b) Identification of discrete LSM};

    \end{tikzpicture}
  		\begin{tikzpicture}[scale=.8, line width=0.5pt, inner sep=0.2mm, shorten >=.1pt, shorten <=.1pt]
		\draw (1.75, 0.8) node(L1) [circle, fill=gray!50,minimum size=0.3cm, draw] {{\footnotesize\,$L_{1}$\,}};

		\draw (0.5, 0) node(L2) [circle, fill=gray!50,minimum size=0.3cm, draw] {{\footnotesize\,$L_2$\,}};
		
		\draw (3, 0) node(L3) [circle, fill=gray!50,minimum size=0.3cm, draw] {{\footnotesize\,$L_3$\,}};

		\draw (1.75, -0.8) node(L4) [circle,fill=gray!50, minimum size=0.3cm, draw] {{\footnotesize\,$L_{4}$\,}};

		%for L1

		\draw (0.25, -0.8) node(X3) [] {{\footnotesize\,$X_4$\,}};
		\draw (1.0, -0.8) node(X4) [] {{\footnotesize\,$X_5$\,}};
        \draw (-.3, -0.8) node(X12) [] {{\footnotesize\,$X_6$\,}};

		%for L2
		\draw (1.25, 1.6) node(X1) [] {{\footnotesize\,$X_1$\,}};

		\draw (2.25, 1.6) node(X2) [] {{\footnotesize\,$X_2$\,}};
		
		\draw (0.45, 1.5) node(X9) [] {{\footnotesize\,$X_3$\,}};

		\draw (2.75, -0.8) node(X5) [] {{\footnotesize\,${X}_{7}$\,}};
		
		\draw (3.3, -0.8) node(X6) [] {{\footnotesize\,${X}_{8}$\,}};

        \draw (3.8, -0.8) node(X11) [] {{\footnotesize\,$X_{9}$\,}};

		%for L3
		\draw (1.0, -1.6) node(X7) [] {{\footnotesize\,$X_{10}$\,}};

		\draw (1.75, -1.6) node(X8) [] {{\footnotesize\,$X_{11}$\,}};
		
		\draw (2.45, -1.6) node(X10) [] {{\footnotesize\,$X_{12}$\,}};

		%for L4

		%
		\draw[-] (L1) -- (L2) node[pos=0.5,sloped,above] {};
		\draw[-] (L1) -- (L3) node[pos=0.5,sloped,above] {};
% 		\draw[->] (L1) -- (L4) node[pos=0.5,sloped,above] {};
        %\draw[->] (L2) -- (L3) node[pos=0.5,sloped,above] {}; 
		\draw[->] (L2) -- (L4) node[pos=0.5,sloped,above] {}; 
		\draw[->] (L3) -- (L4) node[pos=0.5,sloped,above] {};

		\draw[->] (L1) -- (X1) node[pos=0.5,sloped,above] {};
		\draw[->] (L1) -- (X2) node[pos=0.5,sloped,above] {};
		
		\draw[->] (L1) -- (X9) node[pos=0.5,sloped,above] {};

		\draw[->] (L2) -- (X12) node[pos=0.5,sloped,above] {};
		
		\draw[->] (L2) -- (X3) node[pos=0.5,sloped,above] {};
		\draw[->] (L2) -- (X4) node[pos=0.5,sloped,above] {};

		\draw[->] (L3) -- (X5) node[pos=0.5,sloped,above] {};
		\draw[->] (L3) -- (X6) node[pos=0.5,sloped,above] {};
        \draw[->] (L3) -- (X11) node[pos=0.5,sloped,above] {};

		\draw[->] (L4) -- (X7) node[pos=0.5,sloped,above]{};
		\draw[->] (L4) -- (X8) node[pos=0.5,sloped,above] {};
		
		\draw[->] (L4) -- (X10) node[pos=0.5,sloped,above] {};
  \draw (1.5, -2.5) node(ii) [] {{\footnotesize\,(c) Identification result\,}};
		%\draw[color=red,->] (L4) -- (X11) node[pos=0.5,sloped,above] {};
		\end{tikzpicture}

		\caption{An example of discrete latent structure model involving 4 latent variables and 12 observed variables (sub-fig (a)). Here, the red edges form a measurement model, while the blue edges form a structural model. The theoretical result of this paper is shown in sub-fig (c).}% \vspace{-0.4cm}
		\label{fig:illust_simple_example} 
	\end{center}
\end{figure}
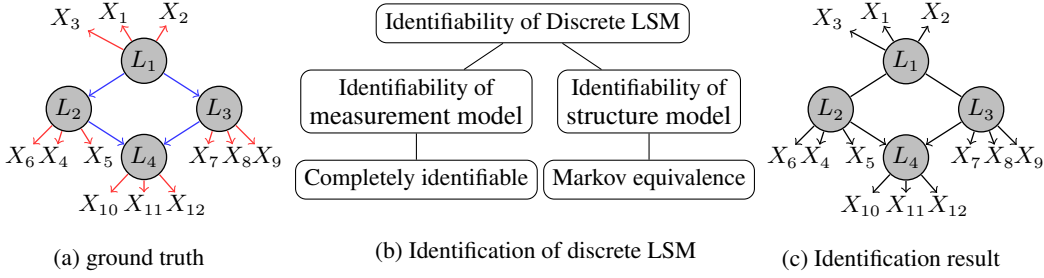

\section{Discrete Latent Structure Model}
% Our work is in the framework of causal graphical models.

For an integer $m$, denote $[m] = \{1,2, \cdots, m\}$. Consider a discrete statistic model with $k$ latent variable set $\mathbf{L} = \{L_1, \cdots, L_k\}, L_i \in [r_i]$ and $m$ discrete observed variable set $\mathbf{X} = \{X_1, \cdots, X_m\}$ with $X_i \in [d_i]$ ($r_i, d_i \geq 2$), in which any marginal probabilities are non-zero. We say a discrete statistic model is a \textit{discrete causal model} if and only if $\mathbf{V} = \mathbf{L} \cup \mathbf{X}$ can be represented by a directed acyclic graph (DAG), denoted by $\mathcal{G}$. We use $\mathrm{supp}(V_i)=\{v\in \mathbb{Z^+}:\mathbb{P}(V_i=v)>0\}$ to denote the set of possible values of the random variable $V_i$. Our work is in the framework of causal graphical models. Concepts used here without explicit definition, such as d-separation, which can refer to standard literature \citep{spirtes2000causation}. 

In this paper, we focus on learning causal structure among latent variables in one class of discrete causal models. The model is defined as follows.

\begin{definition}[Discrete Latent Structure Model]
    A discrete causal model is the Discrete \textbf{L}atent \textbf{S}tructure \textbf{M}odel (Discrete LSM) if it further satisfies the following three assumptions:
    \begin{itemize}[leftmargin=15pt,itemsep=2pt,topsep=0pt,parsep=1pt]
        \item [1)] [Purity Assumption] there is no direct edges between the observed variables;
        \item [2)] [Three-Pure Child Variable Assumption] each latent variable has at least three pure variables as children;
        \item [3)] [Sufficient Observation Assumption] 
        % the number of categories of observed variables is larger than the number of categories of observed variables
        The dimension of observed variables support is larger than the dimension of any latent variables support. 
    \end{itemize}
\end{definition}

These structural constraints inherent in the discrete LSM are also widely used in linear latent variable models, e.g., \cite{Silva-linearlvModel,Kummerfeld2016,cai2019triad,xie2020GIN}. In the binary latent variable case, recently, a similar definition is also employed in \cite{gu2023bayesian,gu2022blessing}. The key difference is that there are no constraints on the latent structure in our work. An example of a discrete LSM model is shown in Fig. \ref{fig:illust_simple_example}(a), where $L_1, \cdots, L_4$ represent discrete latent variables, and $X_1, \cdots, X_{12}$ are discrete observed (measured) variables.

In general, the discrete LSM model can be divided into two sub-models \citep{spirtes2000causation}, i.e., the measurement model and the structure model, e.g., red edge and blue edge in Fig. \ref{fig:illust_simple_example} (a). By this, one can first identify the measurement model to determine the latent variables and then use the measured variable to infer the causal structure of latent variables. As shown in Fig. \ref{fig:illust_simple_example} (b), we will respectively discuss the identification of two sub-models and show that the measurement model is fully identifiable and the structure model is identified up to a Markov equivalence class. The symbols used in our work is summarised in Table \ref{table:notation1}.

To ensure the identification of causal structure and the asymptotic correctness of identification algorithms, some common causal assumptions are required.

% to establish the identifiability of causal structure.

\begin{assumption}[Causal Markov Assumption]\label{ass_1}
    Let $\mathcal{G}$ be a causal graph with vertex set $\mathbf{V}$ and $\mathbb{P}_{\mathbf{V}}$ be probability distribution over the vertices in $\mathbf{V}$ generated by $\mathcal{G}$. We say $\mathcal{G}$ and $\mathbb{P}_{\mathbf{V}}$ satisfy the Causal Markov Assumption if and only if for every $V_i \in \mathbf{V}$, $\mathbb{P}(V_i, \mathbf{V}\setminus \mathrm{Des}_{V_i}|\mathrm{Pa}_{V_i} = i) = \mathbb{P}(W |\mathrm{Pa}_{V_i} = i) \mathbb{P}(\mathbf{V}\setminus \mathrm{Des}_{V_i} |\mathrm{Pa}_{V_i} = i)$.
    
    % $W$ is independent of $\mathbf{V}\setminus \mathrm{Des}_{W} \cup \mathrm{Pa}_{W}$ given $\mathrm{Pa}_{W}$.
\end{assumption}

\begin{assumption}[Faithfulness Assumption]
    Let $\mathcal{G}$ be a causal graph with vertex set $\mathbf{V}$ and $\mathbb{\mathbb{P}_{\mathbf{V}}}$ be probability distribution over the vertices in $\mathbf{V}$ generated by $\mathcal{G}$. We say $<\mathcal{G}, \mathbb{P}_{\mathbf{V}}>$ satisfies the Faithfulness Assumption if and only if (i). every conditional independence relation true in $\mathbb{P}_{\mathbf{V}}$ is entailed by the Causal Markov Assumption applied to $\mathcal{G}$, and (ii). for any joint distribution $\mathbb{P}(\mathbf{L}_p)$, there does not exist $\mathbb{P}(\mathbf{L}_q)$ with $|\mathrm{supp}(\mathbf{L}_q)|<|\mathrm{supp}(\mathbf{L}_p)|$ such that $\mathbb{P}(\mathbf{L}_p) = \mathbb{P}(\mathbf{L}_q)$.

\end{assumption}

\begin{assumption}[Full Rank Assumption]\label{ass_3}
For any conditional probability $\mathbb{P}(X|\mathrm{Pa}_X)$, the corresponding contingency table is full rank, i.e., each column of $\mathbb{P}(X|\mathrm{Pa}_X)$ is linearly independent with the other column vectors in the parameter space.

% for a columns of $\mathbb{P}(X|\mathrm{Pa}_X)$, denoted $\mathcal{T}_{(X|\mathrm{Pa}_X=i)} = \mathbb{P}(X|\mathrm{Pa}_X = i)$, $\mathcal{T}_{(X|\mathrm{Pa}_X=i)} \neq \mathbf{\alpha}^\intercal \mathcal{T}_{(X|\mathrm{Pa}_X=\mathbf{c})}$, where $\mathcal{T}_{(X|\mathrm{Pa}_X=\mathbf{c})}$ is any slice of $\mathbb{P}(X|\mathrm{Pa}_X)$ with $i \notin \mathbf{c}$.

\end{assumption}

The Causal Markov Assumption and Faithfulness Assumption are widely used in the constraint-based causal discovery methods, e.g., PC algorithm and FCI algorithm \citep{spirtes2000causation, spirtes1991PC}. One can see that we further constraint the parameter space of joint distribution cannot be reduced to a low-dimension space, for maintaining the diversity of parameter space. This is also the reason for the full rank assumption, which has also been used in related studies \citep{gu2022blessing}.

\textbf{Our goal:} The target of our work is to answer the identification of the discrete latent structure model, including the measurement model and the structure model.

\begin{table}[htp]
\begin{adjustbox}{width=0.9\textwidth,center}
\renewcommand\arraystretch{1.0}
\tiny
  \begin{tabular}{|r|r|}
  \hline
  $\mathbf{V}$: The set of variables $\mathbf{V} = \mathbf{X} \cup \mathbf{L}$               & $\mathbf{X}$: The set of observed variables            \\
  \hline
  $\mathbf{L}$: The set of latent variables             & $V_i \Vbar V_j | \mathbf{V}_p$ : Conditional independence    \\
  \hline
  $|\mathbf{X}_p|$ : Dimension of $\mathbf{X}_p$             & $\mathcal{T}_{(\mathbf{X}_p)}$ : the tensor form of $\mathbb{P}(\mathbf{X}_p)$     \\
  \hline
  $\mathbb{P}(\mathbf{X}_p)$: the joint distribution of $\mathbf{X}_p$             & $\mathrm{Rank}(\mathcal{T}_{(\mathbf{X}_p)})$ : The rank of tensor $\mathcal{T}_{(\mathbf{X}_p)}$     \\
  \hline
  $\mathrm{Pa}_{X}$: The parent set of $X$ & $\mathrm{Des}_{X}$: The descendant set of $X$    \\
  \hline
   $\mathrm{Diag}(\mathrm{M})$: The diagonal matrix of $\mathrm{M}$ & $\mathbf{u}_i \otimes \mathbf{u}_j$ : The outer product of two vectors    \\
  \hline
  \end{tabular}
\end{adjustbox}
\caption{Mathematical notations used in this paper.}
\label{table:notation1}
\end{table}

\section{Tensor Rank Condition with Graphical Criteria}

To address the identification problem in the discrete LSM, this section introduces the building block--the tensor rank condition of the discrete causal model. Then, we establish the connection between tensor rank and d-separation relations under a general discrete causal model.

Before formalizing the tensor rank condition, we first give the explicit definition of tensor rank.

\begin{definition}[Rank-one Tensor]
An n-way tensor $\mathcal{T} \in \mathbb{R}^{I_1 \times \cdots \times I_n}$ is a rank-one tensor if it can be written as the outer product of $n$ vectors, i.e., 
\[
\mathcal{T} = \mathbf{u}_1 \otimes \mathbf{u}_2 \otimes \cdots \otimes \mathbf{u}_n,
\]
where \( \mathbf{u}_i \) are vectors that each represent a dimension of the tensor, $\otimes$ represents the outer product.

\end{definition}

\begin{definition}[Tensor Rank \cite{kolda2009tensor}]
For an n-way tensor $\mathcal{T} \in \mathbb{R}^{I_1 \times \cdots \times I_n}$, the rank of a tensor $\mathcal{T}$ is defined as the \textbf{smallest} number of rank-one tensors that sum to exactly represent $\mathcal{T}$. Formally, the rank of tensor $\mathcal{T}$, denoted \( \operatorname{rank}(\mathcal{T}) \), is the smallest integer \( r \) such that:
\[
\mathcal{T} = \sum_{i=1}^r \mathbf{u}_1^{(i)} \otimes \mathbf{u}_2^{(i)} \otimes \cdots \otimes \mathbf{u}_n^{(i)},
\]
where each \( \mathbf{u}_k^{(i)} \) is a vector in the corresponding vector space associated with the \( k \)-th mode of \( \mathcal{T} \).
    
\end{definition}

In other words, the tensor rank denotes the minimal number of rank-one decompositions. In the discrete causal model, the joint distribution can be represented as a tensor, e.g., the joint distribution of two random variables is a two-way contingency tensor. Interestingly, by carefully analyzing the rank-one decomposition of the joint distribution, we find that the tensor rank essentially reveals structural information within the causal graph. The result is shown below.

%We thus provide the graphical criteria of tensor rank as follows.

\begin{theorem}
    [Graphical implication of tensor rank condition] \label{the:graph}
    
    In the discrete causal model, suppose Assumption \ref{ass_1} $\sim$ Assumption \ref{ass_3} holds. Consider an observed variable set $\mathbf{X}_p = \{X_1, \cdots, X_n\}$ ($\mathbf{X}_p \subseteq \mathbf{X}$ and $ n\geq 2$) and the corresponding n-way probability tensor $\mathcal{T}_{(\mathbf{X}_p)}$ that is the tabular representation of the joint probability mass function $\mathbb{P}(X_1, \cdots, X_n)$, then $\mathrm{Rank}(\mathcal{T}_{(\mathbf{X}_p)}) = r$ ($r > 1 $) if and only if (i) there exists a variable set $\mathbf{S} \subset \mathbf{V}$ with $|\mathrm{supp}(\mathbf{S})| = r$ that d-separates any pair of variables in $\{X_1, \cdots ,X_n\}$, and (ii) does no exist conditional set $\tilde{\mathbf{S}}$ that satisfies $|\mathrm{supp}(\tilde{\mathbf{S}})| < r$.

\end{theorem}

We further provide an example to illustrate Theorem \ref{the:graph}, and a more comprehensive case is provided in Appendix \ref{sec_exp}.

\begin{example}[Illustrating for the graphical criteria]

Consider a single latent variable structure as shown in Fig. \ref{fig:exp} (a) where $L$ is a latent variable with $\mathrm{supp}(L) = \{0,1\}$ and $X_1, X_2$ are observed variables with $\mathrm{supp}(X_i) = \{0,1,2\}, i \in \{1, 2\}$. For convenience, we denote $p_{i} = \mathbb{P}(L = i)$, $\hat{p}_{i|j} = \mathbb{P}(X_1 =i| L = j)$, and $\tilde{p}_{i|j} = \mathbb{P}(X_2 =i| L_1 = j)$. For the joint distribution of $\mathbb{P}(X_1, X_2)$, it can be represented by the product of conditional probability, as shown in Fig. \ref{fig:exp}(b). By applying the tensor decomposition, one can see that $\mathbb{P}(X_1, X_2)$ can be decomposed as the sum of two rank-one tensors: $\mathbb{P}(X_1, X_2| L=0)$ and $\mathbb{P}(X_1, X_2| L=1)$, i.e., the rank of the tensor $\mathbb{P}(X_1, X_2)$ is two that related to the dimension of the latent support. The reason that $\mathbb{P}(X_1, X_2| L=i)$ is a rank-one tensor is that $L$ d-separates $X_1$ and $X_2$, i.e., $\mathbb{P}(X_1, X_2| L=i) = \mathbb{P}(X_1| L=i)\otimes \mathbb{P}(X_2| L=i)$. This illustrates a connection between tensor rank and the d-separation relations.

\end{example}

Intuitively, the graphical criteria theorem suggests that, in the discrete causal model, the tensor rank condition implies the minimal conditional probability decomposition within the probability parameter space, which hopefully induces the structural identifiability of the discrete LSM model.

\section{Structure Learning of Discrete Latent Structure Model}

In this section, we address the identification problem of the discrete LSM model using a carefully designed algorithm that leverages the tensor rank condition. Specifically, we first show that latent variables can be identified by finding causal clusters among observed variables (Sec. \ref{sec41}). Then, we use these causal clusters to conduct conditional independence tests among latent variables based on the tensor rank condition, identifying the structure model (Sec. \ref{sec42}). Finally, we discuss the practical implementation of testing tensor rank (Sec. \ref{sec_pratest}). For simplicity, we focus on the case where all latent variables have the same number of categories, with extensions provided in Appendix \ref{sec_dif}.

\subsection{Identification of the measurement model}\label{sec41}

To answer the identification of the measurement model, one common strategy is to find the causal cluster that shares the common latent parent, which has been well-studied within the linear model, such as \cite{Silva-linearlvModel,cai2019triad,xie2020GIN}. We follow this strategy and show that, in the discrete LSM, the causal cluster can be found by testing the tensor rank conditions iteratively. The definition of a causal cluster is as follows.

\begin{definition}[Causal cluster]
    In the discrete LSM model, the observed $\{X_1, \cdots, X_n\}$ is a causal cluster, termed $C_i$, if and only if all variables in $\{X_1,\cdots, X_n\}$ share the common latent parent.
\end{definition}

It is not hard to see that, the measurement model can be identified if all causal cluster is found. In order to find these causal clusters by making use of the tensor rank condition, the key issue is to determine the support of latent variables in advance. This issue can be addressed by identifying the rank of the two-way tensor formed by the joint distribution of two observed variables

\begin{proposition}[Identification of support of latent variables]\label{Pro_rank}
    In the discrete LSM model suppose Assumption \ref{ass_1} $\sim$ Assumption \ref{ass_3} holds. The support of the latent variable is the rank of a two-way probability contingency table of any $X_i$ and $X_j$, i.e., $|\mathrm{supp}(L)| = \mathrm{Rank}(\mathcal{T}_{(X_i, X_j)})$, $\forall X_i, X_j \in \mathbf{X}$.    
\end{proposition}

This result holds because any pair of variables in the discrete LSM model is d-separated by any one of their latent parent variables, and all latent variables have the same support. Next, we formalize the property of clusters and give the criterion for finding clusters.

\begin{proposition}[Identification of causal cluster]\label{prop_cluster}
        In the discrete LSM mode, suppose Assumption \ref{ass_1} $\sim$ Assumption \ref{ass_3} holds. Let $r = |\mathrm{supp}(L_i)|$ be the dimension of latent support, for three disjoint observed variables $X_i, X_j, X_k \in \mathbf{X}$, 
\begin{itemize}[leftmargin=15pt,itemsep=2pt,topsep=0pt,parsep=1pt]
            \item $\mathcal{R}ule 1$: if the rank of tensor $\mathcal{T}_{(X_i, X_j, X_k)}$ is not equal to $r$, i.e., $\mathrm{Rank}(\mathcal{T}_{(X_i, X_j, X_k)}) \neq r$, then $X_i$, $X_j$ and $X_k$ belong to the different latent parent.
            \item $\mathcal{R}ule 2$: for any $X_s$, $X_s \in \mathbf{X} \setminus \{X_i, X_j, X_k\}$, the rank of tensor $\mathcal{T}_{(X_i, X_j, X_k, X_s)}$ is $r$, i.e., $\mathrm{Rank}(\mathcal{T}_{(X_i, X_j, X_k, X_s)}) = r$, then $\{X_i, X_j, X_k\}$ share the same latent parent.
        \end{itemize}

\end{proposition}

% \begin{proposition}[Practial Identification method of causal cluster]

%  For $X_i, X_j$ and $X_k$, if the rank of the tensor of them, denoted by $\mathcal{T}$, is not equal to $R$, then $X_i$, $X_j$ and $X_k$ belong to the different latent parent.
    
% \end{proposition}

\begin{example}[Finding causal clusters]
    Let's take Fig. \ref{fig:illust_simple_example}(a) as an example. One can find that for $\{X_1, X_2, X_3, X_k\}$, where $X_k \in \mathbf{X} \setminus \{X_1, X_2, X_3\}$, the rank of tensor $\mathcal{T}_{(X_1, X_2, X_3, X_k)}$ is $r$. Thus, $\{X_1, X_2, X_3\}$ is identified as a causal cluster.
\end{example}

Next, we consider the practical issues involved in determining the number of latent variables by causal clusters. That is, there are some causal clusters that should be merged because they share one latent parent. We find that the overlapping clusters can be directly merged into one cluster. This is because the overlapping clusters have the same latent variable as the parent under the discrete LSM model. The validity of the merge step is guaranteed by Proposition \ref{pro_merger}.

\begin{proposition}[Merging Rule]\label{pro_merger}
In the discrete LSM model, for two causal clusters $C_1$ and $C_2$, if $C_1 \cap C_2 \neq \emptyset$, then $C_1$ and $C_2$ share the same latent parent.
\end{proposition}

Based on the above results, one can iteratively identify causal clusters and apply the merger rule to detect all latent variables. The identification procedure is summarized in Algorithm 1.

\begin{algorithm}[h]
    \caption{Finding the causal cluster}
    \label{alg:algorithm_1}
    \textbf{Input}: Data from a set of measured variables $\mathbf{X}_{\mathcal{G}}$, and the dimension of latent support $r$\\
    %\textbf{Parameter}: Optional list of parameters\\
    \textbf{Output}: Causal cluster $\mathcal{C}$
    \begin{algorithmic}[1] %[1] enables line numbers
    	\STATE Initialize the causal cluster set $\mathcal{C} \coloneqq \emptyset$, and $\mathcal{G}^\prime= \emptyset$;
        \STATE \Alcomment{\textit{// Identify Causal Skeleton}}
        \STATE \textbf{Begin} the recursive procedure 
        \REPEAT{
        \FOR{each $X_i, X_j$ and $X_k$ $\in \mathbf{X}$}
        \IF{$\mathrm{Rank}(\mathcal{T}_{\{X_i, X_j, X_k\}}) \neq r$}
        \STATE \textbf{Continue};  ~~\Alcomment{\textit{// $\mathcal{R}ule 1$ of Prop. \ref{prop_cluster}}}
        \ENDIF
        \IF{$\mathrm{Rank}(\mathcal{T}_{\{X_i, X_j, X_k, X_s\}})=r$, for all $X_s \in \mathbf{X} \setminus \{X_i, X_j, X_k\}$}
        \STATE $\mathbf{C} = \mathbf{C} \cup \{\{X_i, X_j, X_k\}\}$;
        \ENDIF
        \ENDFOR
        }
        \UNTIL{no causal cluster is found.}
        \STATE \Alcomment{\textit{// Merging cluster and introducing latent variables}}
        \STATE Merge all the overlapping sets in $\mathbf{C}$ by Prop. \ref{pro_merger}.
        \FOR{each $C_i \in \mathbf{C}$}
        \STATE Introduce a latent variable$L_i$ for $C_i$;
        \STATE $\mathcal{G} = \mathcal{G} \cup \{L_i \to X_j|X_j \in C_i \}$.
        \ENDFOR
        \STATE \textbf{return} Graph $\mathcal{G}$ and causal cluster $\mathcal{C}$.
    \end{algorithmic}
\end{algorithm}

\begin{theorem}[Identification of the measurement model]\label{the:mm}
    In the discrete LSM model, suppose Assumption \ref{ass_1} $\sim$ Assumption \ref{ass_3} hold. The measurement model is fully identifiable by Algorithm 1.
\end{theorem}

\subsection{Identification of the structure model}\label{sec42}

Once the measurement model is identified, the observed children can serve as proxies for the latent variables, enabling the identification of the causal structure among them. Here, we employ constraint-based framework to learn the causal structure of latent variables.

Constraint-based structure learning algorithms find the Markov equivalence class over a set of variables by making decisions about independence and conditional independence among them. Given a pure and accurate measurement model with at least two measures per latent variable, we can test for independence and conditional independence (CI) among the latent variables. Specifically, to test statistical independence between discrete variables, one can examine whether the rank of their joint distribution contingency table is one \citep{sullivant2018algebraic}. For testing conditional independence (CI) relations among latent variables, further leveraging the algebraic properties of the tensor rank condition is required (see Theorem \ref{the:d-separation}).

 % This enables the search for equivalence classes of structural models among them.

\begin{theorem}[d-separation among latent varaible]\label{the:d-separation}
     In the discrete LSM model, suppose Assumption \ref{ass_1} $\sim$  Assumption \ref{ass_3} hold.  Let $r$ be the dimension of latent support, then $L_i \bot L_j | \mathbf{L}_p$ if and only if $\mathrm{Rank}(\mathcal{T}_{(X_i, X_j, \mathbf{X}_{p1}, \mathbf{X}_{p2})}) = r^{|\mathbf{L}_p|}$, where $X_i$ and $X_j$ are the pure children of $L_i$ and $L_j$, $\mathbf{X}_{p1}$ and $\mathbf{X}_{p2}$ are two disjoint child set of $\mathbf{L}_p$ that satisfy $\forall L_i \in \mathbf{L}_p, \mathrm{Ch}_{L_i} \cap \mathbf{X}_{p1} \neq \emptyset, \mathrm{Ch}_{L_i} \cap \mathbf{X}_{p2} \neq \emptyset$.

     % , with $|\mathbf{X}_{p1}| = |\mathbf{X}_{p2}| = |\mathbf{L}_p|$.
\end{theorem}

Intuitively, based on the graphical criteria of tensor rank condition, $\mathbf{L}_p$ is a minimal conditional set in the causal graph that d-separates $\mathbf{X}_{p1}$ and $\mathbf{X}_{p2}$ and hence the rank of tensor $\mathcal{T}_{(X_i, X_j, \mathbf{X}_{p1}, \mathbf{X}_{p2})}$ is the dimension of support of $\mathbf{L}_p$, if $X_i$ and $X_j$ also be d-separated by $\mathbf{L}_p$. An illustrative example is given below.

\begin{example}[CI test among latent variables]
    Consider the structure in Fig. \ref{fig:illust_simple_example}(a). Suppose $r = 2$. By selecting $\mathbf{X}_{p1} = \{ X_4, X_7\}$ and $\mathbf{X}_{p2} =\{X_5, X_8\}$ to be two disjoint child set of $\{L_2, L_3\}$ respectively, let $\mathbf{X}_p = \{X_1, X_{10}, X_4, X_5, X_7, X_8\}$, one can see that the rank of tensor $\mathcal{T}_{(\mathbf{X}_p)}$ is four (due to $\{L_2, L_3\}$ is minimal conditional set that d-separates any pair variable in $\mathbf{X}_p$), which imply that $L_1 \Vbar L_4 |\{L_2, L_3\}$.
\end{example}

Based on Theorem \ref{the:d-separation}, we introduce the PC-TENSOR-RANK algorithm. This method accepts a measurement model learned by the previous procedure, and outputs the Markov equivalence class of the structural model associated with the latent variables within the measurement model, in accordance with the PC algorithm. The implementation is summarised as Algorithm 2. Consequently, we establish the identification of the structure model as shown in Theorem \ref{the_stru}.

\begin{algorithm}[h]
    \caption{PC-TENSOR-RANK}
    \label{alg:algorithm_3}
    \textbf{Input}: Data set $\mathbf{X} = \{ {X_1},\dots,{X_m}\} $ and causal cluster $\mathcal{C}$\\
    %\textbf{Parameter}: Optional list of parameters\\
    \textbf{Output}: A partial DAG $\mathcal{G}$.
    \begin{algorithmic}[1] %[1] enables line numbers
    \STATE Initialize the maximal conditions set dimension $k = 2$;
    \STATE Let $L_i$ denote as $C_i$, $C_i \in \mathcal{C}$;
    \STATE Form the complete undirected graph $\mathcal{G}$ on the latent variable set ${\mathbf{L}}$;
    \FOR{$\forall L_i, L_j\in {\mathbf{L}}$ and adjacent in $\mathcal{G}$}
    \STATE \Alcomment{\textit{//Test the CI relations among latent variables by Theorem \ref{the:d-separation}}}
    \IF{$\exists \mathbf{L}_p\subseteq {\mathbf{L}} \setminus \{L_i,L_j\}$ and ($|\mathbf{L}_p|< k$) such that $L_i \Vbar L_j |  \mathbf{L}_p$ hold}
    \STATE delete edge $L_i-L_j$ from $G$;
    % \STATE record $\mathbf{L}_p$ in $Sepset(L_i,L_j)$ and $Sepset(L_j,L_i)$;
    \ENDIF
    \ENDFOR
    \STATE Search V structures and apply meek rules \cite{meek1995causal}.
    \STATE \textbf{return} a partial DAG $\mathcal{G}$ of latent variables.
    \end{algorithmic}
\end{algorithm}

% Consequently, we establish the following result.

\begin{theorem}[Identification of structure model]\label{the_stru}
In the discrete LSM model, suppose Assumption \ref{ass_1} $\sim$ Assumption \ref{ass_3} holds. Given the measurement model, the causal structure over the latent variable is identified up to a Markov equivalent class by the PC-TESNOR-RANK algorithm.
    
\end{theorem}

\subsection{Practical Test for Tensor Rank}\label{sec_pratest}

In our theoretical results, the key issue is to test the rank of a tensor, which involves estimating the dimension of latent support and the rank of a tensor. Here, we aim to explore methods to (i) estimate the rank of the contingency matrix for determining the dimension of latent support, and (ii) apply the goodness-of-fit test to assess the tensor rank.

\paragraph{Estimate the rank of contingency matrix.}
We start with the estimation of the dimension of latent variables support based on Prop. \ref{Pro_rank}. There are many practical approaches used to estimate the rank of a general matrix $\mathrm{M}$, such as \cite{camba2009statistical}. In our implementation, we use the characteristic root statistic, abbreviated as CR statistic \cite{robin2000tests}, to test the rank of the probability contingency matrix of two observed variables. Specifically, Let $\tilde{\mathrm{M}}$ be an asymptotically normal estimator of $\mathrm{M}$, then the CR statistic is the sum of $d-r$ smallest singular values of $\tilde{\mathrm{M}}$, multiplied by the sample size. Under the null hypothesis, the above statistic converges in distribution to a weighted (given by the eigenvalues) sum of independent $\mathcal{X}^{2}_{1}$ random variables \cite{robin2000tests}.

% \subsubsection{Goodness-of-fit test for tensor rank}\label{sec_rank2}

\paragraph{Goodness-of-fit test for tensor rank.}

Once the dimension of the support of latent variables is identified, in the structure learning procedure, we perform the following hypotheses test: $\mathcal H_0$: $\mathrm{Rank}(\mathcal{T})=r$ \textit{v.s.} $\mathcal H_1$: $\mathrm{Rank}(\mathcal{T})\neq r$. To achieve this, we first apply the CP decomposition technology to the target tensor $\mathcal{T}$ as a sum of $r$ rank-one tensors given specified $r$, then we evaluate how well the reconstructed tensor from this decomposition approximates the original tensor to conduct the hypotheses test.

To perform the rank-decomposition with specified $r$ on the probability contingency tensor, one can use the non-negative CP decomposition to decompose the tensor into the sum of $r$ rank-one tensor \cite{shashua2005non}. Given the decomposition, one can obtain a reconstructed tensor, denoted by $\tilde{\mathcal{T}}$, from the outer product of decomposed vectors.

With the reconstructed tensor, we constructed square-chi goodness of fit test \cite{cochran1952chi2} for testing $\mathrm{Rank}(\mathcal{T})=r$. Such a test is frequently used to summarize the discrepancy between observed values and the expected values, which measure the sum of differences between observed and expected outcome frequencies. Let $\textbf{vec}(\mathcal{T})$ be the vectorization of tensor $\mathcal{T}$, suppose $\textbf{vec}(\tilde{\mathcal{T}})$ be the asymptotic normality estimator of $\textbf{vec}(\mathcal{T})$, we have the chi-square statistic as 
    $\mathcal{X}^2 = \sum_{i \in \textbf{vec}(\tilde{\mathcal{T}})} \frac{(\textbf{vec}(\mathcal{T})_i -\textbf{vec}(\tilde{\mathcal{T}})_i)^2}{\textbf{vec}(\tilde{\mathcal{T}})_i}$, which follows the $\mathcal{X}^2$ distribution with freedom degrees $ \prod_{i \in [n]} d_i - (\sum_{i,j \in [n]} d_id_j)$.

\section{Simulation Studies}

In this section, we conducted simulation studies to assess the correctness of the proposed methods in causal structure learning tasks. The baseline approaches include Building Pure Cluster (BPC) \cite{Silva-linearlvModel}, Latent Tree Model (LTM) \cite{choi2011latenttree}, and Bayesian Pyramid Model (BayPy) \cite{gu2023bayesian}.

% \begin{figure}[t]
%     \centering
%     \includegraphics[width=0.95\textwidth]{fig/simulation.png}
%     \caption{The Structural and Measurement models used in our simulation studies.}
%     \label{fig:simulated}
% \end{figure}

In the following simulation studies, we consider the different combinations of various types of structure models(SM) and measurement models(MM). Specifically, for the structure model, we consider the following five typical cases: [SM1]: $L_1 \to L_2$; [SM2]: ${L_1 \to L_2 \to L_3}$; [SM3]: the structure of latent variables is shown in Fig. \ref{fig:illust_simple_example}(a); [Collider]: $L_1 \to L_2 \leftarrow L_3$; [Star]: $L_1 \to L_2, L_1 \to L_3, L_1 \to L_4$. For the measurement model, we consider the following two cases: [MM1]: each latent variable has three pure observed variables, i.e., $L_i \to \{X_1, X_2, X_3\}$; [MM2]: each latent variable has four pure observed variables, i.e., $L_i \to \{X_1, X_2, X_3, X_4\}$. 

In all cases, the data generation process follows the discrete LSM model: (i) we generate the probability contingency table of latent variables in advance, according to different latent structures (e.g., SM1), then (ii) we generate the conditional contingency table of observed variables (condition on their latent parent), and finally (iii) we sample the observed data according to the probability contingency table, where the dimension of latent support $r$ is set to 3 and the dimension of all observed variables support is set to 4, sample size ranged from $\{5k, 10k, 50k\}$.

For each simulation study, we randomly generate the dataset and apply the proposed algorithm and baselines to these data. We use the following scores for evaluating the performance of causal clusters from each algorithm: \textbf{latent omission}, \textbf{latent commission}, and \textbf{mismeasurement}. Moreover, to assess the ability of these algorithms to correctly discover the causal structure among latent variables, we use the metric like \textbf{edge omission (EO)}, \textbf{edge commission (EC)}, and \textbf{orientation omission (OO)}. These metrics can be referred to \cite{Silva-linearlvModel}, in which the tasks are aligned with our work. Each experiment
was repeated ten times with randomly generated data, and the results were averaged.

% These metric can be refered to \cite{Silva-linearlvModel}, in which the tasks are aligned with our work.

% \begin{itemize}
%     \item \textbf{latent omission}, the number of latents in $G$ that do not appear in $G_{out}$ divided by the total number of true latents in $G$;
%     \item \textbf{latent commission}, the number of latents in $G_{out}$ that could not be mapped to a latent in $G$ divided by the total number of true latents in $G$;
%     \item \textbf{mismeasurement}, the number of observed variables in $G_{out}$ that are measuring at least one wrong latent divided by the number of observed variables in $G$;
% \end{itemize}

\begin{center}
\begin{table*}[htp!]
% \centering\color{red}
%  \vspace{-3mm}
% \setlength{\abovecaptionskip}{1pt}
% 	\setlength{\belowcaptionskip}{1pt}
	\small
	\center \caption{Results on learning pure measurement models, where the data is generated by the discrete LSM. Lower value means higher accuracy.}
% 	\vspace{-3mm}
	\label{tab:mixed distribution}
	\resizebox{1.0\textwidth}{!}{
	\begin{tabular}{cccccccccccccc}
		\hline  \multicolumn{2}{c}{} &\multicolumn{4}{c}{\textbf{Latent omission}} & \multicolumn{4}{c}{\textbf{Latent commission}} & \multicolumn{4}{c}{\textbf{Mismeasurements}}\\
		%\hline 
		\multicolumn{2}{c}{Algorithm} & \textbf{Our} & BayPy  & LTM & BPC & \textbf{Our} & BayPy & LTM & BPC & \textbf{Our} & BayPy & LTM & BPC \\
		\hline 
		 & 5k & 0.15(3) & 0.10(2) & 0.15(3) & 0.96(10) 
        		 & 0.00(0) & 0.10(2)& 0.00(0) & 0.00(0) 
        		 & 0.05(1) & 0.00(0) & 0.00(0) & 0.00(0) \\
		% \cline{2-14}
		{\emph{$SM_1 + MM_1$}} &10k &  0.05(1) &0.05(1) & 0.10(2) & 0.90(10) 
                		& 0.00(0)&0.05(1) & 0.00(0) & 0.00(0) 
                		& 0.00(0) & 0.00(0)& 0.00(0) & 0.00(0) \\
		% \cline{2-14}
		&50k & 0.00(0) & 0.00(0) & 0.00(0) & 0.90(10) 
        		& 0.00(0) & 0.00(0)& 0.00(0) & 0.00(0) 
        		& 0.00(0) & 0.00(0) & 0.00(0) & 0.00(0) \\
		\hline 
		 %case 2
		 & 5k 
		 & 0.23(5) & 0.19(6) & 0.26(6) & 0.90(10) 
		 & 0.00(0) & 0.19(6) & 0.03(1) & 0.00(0)
		 & 0.05(2) & 0.19(6) & 0.23(6) & 0.00(0)  \\
		% \cline{2-14}
		{\emph{$SM_2 + MM_1$}} &10k 
		& 0.13(4) & 0.13(4) & 0.13(4) & 0.86(10) 
		& 0.00(0) & 0.03(4) & 0.00(0) & 0.00(0) 
		& 0.00(0) & 0.13(4) & 0.13(4) & 0.00(0)\\
		% \cline{2-14} 
            &50k
		& 0.06(2) & 0.10(3) & 0.10(3) & 0.86(10)
		& 0.00(0) & 0.13(4) & 0.00(0) & 0.00(0)
		& 0.00(0) & 0.13(4) & 0.10(3)  & 0.00(0)\\
		\hline 
		%case 3
		 & 5k 
		 & 0.12(2) & 0.19(6) & 0.21(5) & 0.90(10) 
		 & 0.00(0) & 0.19(6) & 0.00(0) & 0.00(0) 
		 & 0.03(1) & 0.16(6) & 0.21(5) & 0.00(0)\\
		% \cline{2-14}
		{\emph{$SM_2 + MM_2$}} &10k 
		& 0.03(1) & 0.13(4) & 0.10(3) & 0.86(10) 
		& 0.00(0) & 0.13(4) & 0.00(0) & 0.00(0) 
		& 0.00(0) & 0.11(4) &0.10(3)  & 0.00(0)\\
		% \cline{2-14}
            &50k 
		& 0.00(0) & 0.07(2) & 0.07(2) & 0.83(10)
		& 0.00(0) & 0.07(2) & 0.00(0) & 0.00(0)
		& 0.00(0) & 0.07(2) & 0.06(2) & 0.00(0)\\
		\hline 
		 %case 4 
		 & 5k
		 &0.25(6) &0.30(6) & 0.55(10) & 0.86(10)
		 &0.00(0) &0.30(6) & 0.00(0) & 0.00(0)
		 &0.12(5) &0.20(6) & 0.55(10) & 0.00(0)\\
		% \cline{2-14}
		{\emph{$SM_3 + MM_1$}} &10k
	     &0.17(5) &0.25(5) & 0.50(10) & 0.83(10)
		 &0.00(0) &0.25(5) & 0.00(0) & 0.00(0)
		 &0.05(3) &0.16(5) & 0.50(10) & 0.00(0)\\
		% \cline{2-14}
		&50k
		 &0.08(3) &0.20(4) & 0.50(10) & 0.83(10)
		 &0.00(0) &0.20(4) & 0.00(0) & 0.00(0)
		 &0.03(2) &0.13(4) & 0.50(10) & 0.00(0)\\    
		\hline 
	\end{tabular}}
% 	\vspace{-1mm}
\label{Tab_1}
\end{table*}
\end{center}

% \begin{itemize}
%     \item \textbf{edge omission (EO)}, the number of edges in the structural model of $G$ that do not appear in $G_{out}$ divided by the possible number of edge omissions;
%     \item \textbf{edge commission (EC)}, the number of edges in the structural model of $G_{out}$ that do not exist in $G$ divided by the possible number of edge commissions;
%     \item \textbf{orientation omission (OO)}, the number of arrows in the structural model of $G$ that do not appear in $G_{out}$ divided by the possible number of orientation omissions in $G$;
% \end{itemize}

\begin{center}
\begin{table*}[htp!]
% \centering\color{red}
%  \vspace{-3mm}
% \setlength{\abovecaptionskip}{1pt}
% 	\setlength{\belowcaptionskip}{1pt}
	\small
	\center \caption{Results on learning the structure model. The symbol '-' indicates that the current method does not output this information. Lower value means higher accuracy.}
% 	\vspace{-3mm}
	\label{tab:mixed distribution}
	\resizebox{1.0\textwidth}{!}{
	\begin{tabular}{cccccccccccccc}
		\hline  \multicolumn{2}{c}{} &\multicolumn{4}{c}{\textbf{Edge omission}} & \multicolumn{4}{c}{\textbf{Edge commission}} & \multicolumn{4}{c}{\textbf{Orientation omission}}\\
		%\hline 
		\multicolumn{2}{c}{Algorithm} & \textbf{Our} & BayPy  & LTM & BPC & \textbf{Our} & BayPy & LTM & BPC & \textbf{Our} & BayPy & LTM & BPC \\
		\hline 
		 & 5k & 0.00(0) & 1.00(10) & 0.26(8) & 1.00(10) 
        		 & 0.10(1) & 0.00(0)& 0.00(0) & 0.00(0) 
        		 & 0.10(1) & 1.00(10) & -- & 1.00(0) \\
		% \cline{2-14}
		{\emph{Collider+$MM_1$}} &10k &  0.00(0) &1.00(10) & 0.23(6) & 1.00(10) 
                		& 0.00(0) &0.02(1) & 0.0(0) & 0.00(0) 
                		& 0.00(0) & 1.00(10) & -- & 1.00(0) \\
		% \cline{2-14}
		&50k & 0.00(0) & 1.00(10) & 0.10(3) & 1.00(10) 
        		& 0.00(0) & 0.00(0)& 0.00(0) & 0.00(0)
        		& 0.00(0) & 1.00(10) & -- & 1.00(0) \\
		\hline 
		 %case 2
		 & 5k 
		 & 0.15(3) & 1.00(10) & 0.16(6) & 1.00(10) 
		 & 0.10(1) & 0.00(0)& 0.00(0) & 0.00(0)
		 & 0.00(0) & 0.00(0) & -- & 0.00(0)  \\
		% \cline{2-14}
		{\emph{$SM_2 + MM_1$}} &10k 
		& 0.05(1) & 1.00(10) & 0.13(4) & 1.00(10) 
		& 0.01(1) & 0.00(0)& 0.00(0) & 0.00(0) 
		& 0.00(0) & 0.00(0) & -- & 0.00(0) \\
		% \cline{2-14} 
            &50k
		& 0.00(0) & 1.00(10) & 0.10(3) & 1.00(10)
		& 0.00(0) & 0.00(0)& 0.00(0) & 0.00(0)
		& 0.00(0) & 0.00(0) & -- & 0.00(0) \\
		\hline 
		%case 3
		 & 5k 
		 & 0.10(3) & 1.00(10) & 0.25(5) & 1.00(10) 
		 & 0.20(5) & 0.00(0)& 0.00(0) & 0.00(0)
		 & 0.00(0) & 0.00(0) & -- & 0.00(0) \\
		% \cline{2-14}
		{\emph{$Star + MM_1$}} &10k 
		& 0.06(2) & 1.00(10) & 0.15(3) & 1.00(10) 
		& 0.08(3) & 0.00(0)& 0.00(0) & 0.00(0) 
		& 0.00(0) & 0.00(0) & -- & 0.00(0) \\
		% \cline{2-14}
            &50k 
		& 0.03(1) & 1.00(10) & 0.15(3) & 1.00(10)
		& 0.05(2) & 0.00(0)& 0.00(0) & 0.00(0)
		& 0.00(0) & 0.00(0) & -- & 0.00(0) \\
		\hline 
		 %case 4 
		 & 5k
		 &0.22(7) &1.00(10)& 0.50(10) & 1.00(10)
		 &0.40(6) & 0.00(0)& 0.02(1) & 0.00(0)
		 &0.20(2) &1.00(10) & -- & 1.00(10)\\
		% \cline{2-14}
		{\emph{$SM_3 + MM_1$}} &10k
	     &0.15(5) &1.00(10) & 0.50(10) & 1.00(10)
		 &0.10(2) & 0.00(0)& 0.00(0) & 0.00(0)
		 &0.10(1) &1.00(10) & -- & 1.00(10)\\
		% \cline{2-14}
		&50k
		 &0.05(2) &1.00(10) & 0.50(10) & 1.00(10)
		 &0.05(1) & 0.00(0)& 0.00(0) & 0.00(0)
		 &0.00(0) &1.00(10) & -- & 1.00(10)\\    
		\hline 
	\end{tabular}}
% 	\vspace{-1mm}
\label{tab_2}
\end{table*}
\end{center}

The results are reported in Table \ref{Tab_1} and Table \ref{tab_2}. Our method consistently delivers the best outcomes across most scenarios, demonstrating its capability to identify both the causal clusters and the causal structures of latent variables. In contrast, the BPC approach performs poorly, as it is specifically designed for linear models. Additionally, the LTM and BayPy algorithms show suboptimal performance in structure learning of latent variables due to their limitations to specific structural models, such as tree structures, or assumptions that latent variables are binary. More experimental results and discussions are provided in the appendix.

\section{Real Data Applications}

We now briefly present the results from two real datasets. The first is the political efficacy dataset, collected by \cite{aish1990panel} through a cross-national survey designed to capture information on both conventional and unconventional forms of political participation in industrial societies. This dataset includes 1719 cases obtained in a USA sample. The second dataset, referred to as the depress dataset, is detailed by \cite{joreskog1996lisrel} and comprises twelve observed variables grouped into three latent factors: self-esteem, depression, and impulsiveness, with a total of 204 samples. Our algorithm learns the correct causal structure (including the measurement model and the structure model) for both datasets by first identifying the dimension of latent support as two in the political efficacy dataset and four in the depress dataset. See the appendix for more details.

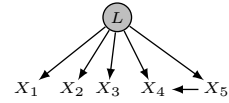
\begin{wrapfigure}{r}{0.3\textwidth} 
\vspace{-16pt}
  \begin{center}
  \begin{tikzpicture}[scale=1.2, line width=0.5pt, inner sep=0.2mm, shorten >=.1pt, shorten <=.1pt]
		\draw (2.2, 2.4) node(L1) [circle, fill=gray!50,minimum size=0.38cm,,draw] {{\tiny\,$L$\,}};
  %fill=gray!60
		%third level
		\draw (1.2,  1.6) node(X1) [] {{\tiny\,$X_1$\,}};
        \draw (1.7,  1.6) node(X2) [] {{\tiny\,$X_2$\,}};
		\draw (2.1,  1.6) node(X3) [] {{\tiny\,$X_3$\,}};
    \draw (2.6,  1.6) node(X4) [] {{\tiny\,$X_4$\,}};
    \draw (3.3,  1.6) node(X5) [] {{\tiny\,$X_5$\,}};

		\draw[-latex] (L1) -- (X1) node[pos=0.5,sloped,above] {};
		\draw[-latex] (L1) -- (X2) node[pos=0.5,sloped,above] {};
        \draw[-latex] (L1) -- (X3) node[pos=0.5,sloped,above] {};
        \draw[-latex] (L1) -- (X4) node[pos=0.5,sloped,above] {};
        \draw[-latex] (L1) -- (X5) node[pos=0.5,sloped,above] {};
        \draw[-latex] (X5) -- (X4) node[pos=0.5,sloped,above] {};
        %\draw [-latex] (X4) edge[bend right=30] (X5);
    \end{tikzpicture}
    \caption{Example of the impure structure that can be identified by tensor rank condition.}
    \label{fig:impure}
  \end{center}
  \vspace{-15pt}
  \vspace{1pt}
\end{wrapfigure} 

\section{Discussion and Further Work}

The preceding sections presented how to use tensor rank conditions to locate the latent variables and identify their causal structure in the discrete LSM model. In this section, we examine whether the impure structure (e.g., an edge between observed variables) can be detected through tensor rank conditions. For instance, consider the structure shown in Fig. \ref{fig:impure}. One can observe that for any subset $\{X_i, X_j, X_k\} \subset \{X_1, X_2, X_3, X_4, X_5\}$ where $\{X_4, X_5\} \not\subset \{X_i, X_j, X_k\}$, one has $\mathrm{Rank}(\mathcal{T}_{(X_i, X_j, X_k)}) = |\mathrm{supp}(L)|$. This implies $\{X_1, X_2, X_3\}$ and $X_4$ (or $X_5$) share the common latent parent. Meanwhile, we have $\mathrm{Rank}(\mathcal{T}_{(X_4, X_5)}) = |\mathrm{supp}(X_5)|$. Moreover, we have $\mathrm{Rank}(\mathcal{T}_{(X_2, X_3, X_4, X_5)}) = |\mathrm{supp}(L)|\cdot|\mathrm{supp}(X_5)|$. By the graphical criteria of tensor rank condition, one can infer that there is an edge between $X_4$ and $X_5$, indicating the impure structure can be identified by the tensor rank condition. Developing an efficient algorithm to learn a more general discrete LSM structure that allows impure structure in a principled way is part of our future work.

\section{Conclusion}

We derive a nontrivial algebraic property of a particular type of discrete causal model under proper causal assumptions. We build the connection between tensor rank and the d-separation relations in the causal graph and propose the graphical criteria of tensor rank. By this, the identifiability of causal structure in discrete latent structure models is achieved based on which we proposed an identification to locate latent causal variables and identify their causal structure. We provide a practical test approach for testing the tensor rank and verifying the efficientness of the proposed algorithm via the simulated studies. The proposed theorems and the algorithms take a meaningful step in understanding the causal mechanism of discrete data. Future research along this line includes allowing casual edges between observed variables and allowing hierarchical latent structure.

\clearpage
\normalem
\bibliography{main}
\bibliographystyle{unsrtnat}

\clearpage
\appendix

{\large\bf Supplementary Material}

The supplementary material contains

\begin{itemize}
    \item Graphical Notations;
    \item Example of Tensor Representations of Joint Distribution;
    \item Discussion of Our Assumptions;
    \item Proofs of Main Results;
    \begin{itemize}
        \item Proof of Theorem \ref{the:graph};
        \item Proof of Proposition \ref{Pro_rank};
        \item Proof of Proposition \ref{prop_cluster};
        \item Proof of Proposition \ref{pro_merger};
        \item Proof of Theorem \ref{the:mm};
        \item Proof of Theorem \ref{the:d-separation};
        \item Proof of Theorem \ref{the_stru}.
    \end{itemize}
    \item Extension of Different Latent State Space;
    \item Discussion with the Hierarchical Structures;
    \item Practical Estimation of Tensor Rank;
    \item More Experimental Results;
    \item More Details of Real-world Dataset;
    
\end{itemize}

\section{Graphical Notations}

% \begin{table}[h]
% \centering
% \small
% \begin{tabular}{c|c}
% \toprule
% Notations & Definitions and Descriptions \\ 
% \midrule
% $\mathbf{V}$    & The set of variables, i.e., $\mathbf{V} = \mathbf{X} \cup \mathbf{L}$ \\
% $\mathbf{X}$    & A set of observed variables  \\
% $\mathbf{L}$    & A set of latent variables \\
% $\mathbf{X}_p$    & A subset of observed variables of $\mathbf{X}$, i.e., $\mathbf{X}_p \subset \mathbf{X}$ \\
% $\mathbf{L}_p$    & A subset of latent variables of $\mathbf{L}$, i.e., $\mathbf{L}_p \subset \mathbf{X}$ \\
% $V_i \Vbar V_j | \mathbf{V}_p$      &$V_i$ and $V_j$ are conditional independent by $\mathbf{V}_p$\\
% $|\mathbf{X}_p|$     & Dimension of $\mathbf{X}_p$ \\
% $\mathcal{T}_{\{X_1,...,X_n\}}$     & The tensor representation of joint probability over $X_1 ... X_n$\\
% $\mathbb{P}(X|L)$     & The conditional distribution of $X$ given $L$\\
% $\mathbf{u}_i \otimes \mathbf{u}_j$     & The outer product for vector $\mathbf{u}_i$ and $\mathbf{u}_j$\\
% $\mathrm{Rank}(\mathcal{T}_{\mathbf{X}_p})$     & The rank of tensor $\mathcal{T}_{\mathbf{X}_p}$\\
% $Pa_{X}$     & The parent set of $X$.\\
% $Des_{X}$     & The descendants set of $X$.\\
% $Anc_{X}$     & The ancestor variable set of $X$.\\
% $\mathrm{Diag}(M)$     & The diagonal matrix $M$.\\
% $\mathbf{u}_i \odot \mathbf{u}_j$     & The Khatri–Rao product of two vectors.\\
% \bottomrule
% \end{tabular}
% \caption{Mathematical notations used in this paper.}
% \label{table:notation}
% \end{table}

Below, we provide some graphical notation used in our work, which is mainly derived from the \cite{pearl2009causality,spirtes2000causation}.

\begin{definition}[Path and Directed Path]
In a DAG, a \textbf{path} $P$ is a sequence of nodes $(V_1,...V_r)$ such that $V_i$ and $V_{i+1}$ are adjacent in $\mathcal{G}$, where $1 \le i<r$. Further, we say a path $P=( V_{i_0} ,V_{i_1} ,\dots ,V_{i_{k}})$ in $G$ is a \textbf{directed path} if it is a sequence of nodes of $G$ where there is a directed edge from $V_{i_j}$ to $V_{i_{(j+1)}}$ for any $0\leq j\leq k-1$.

\end{definition}

\begin{definition}[Collider]
A \textbf{collider} on a path $\{V_1,...V_p\}$ is a node $V_i$ , $1< i < p$, such that $V_{i-1}$ and $V_{i+1}$ are parents of $V_{i}$.
\end{definition}

Graphically, we also say a collider is a `V-structure'.

\begin{definition}[d-separation]
A path $p$ is said to be d-separated (or blocked) by a set of nodes $\mathbf{Z}$ if and only if the following two conditions hold:
\begin{itemize}
    \item $p$ contains a chain $V_i \to V_k \to V_j$ or a fork $V_i \leftarrow V_k \rightarrow V_j$ such that the middle node $V_k$ is in $\mathbf{Z}$; 
    \item $p$ contains a collider $V_i \to V_k \leftarrow V_j$ such that the middle node $V_k$ is not in $\mathbf{Z}$ and such that no descendant of $V_k$ is in $\mathbf{Z}$.
\end{itemize}
\end{definition}
A set $\mathbf{Z}$ is said to d-separate $\mathbf{A}$ and $\mathbf{B}$ if and only if $\mathbf{Z}$ blocks every path from a node in $\mathbf{A}$ to a node in $\mathbf{B}$. We also denote as $\mathbf{A} \Vbar \mathbf{B} |\mathbf{Z}$ in the causal graph model.

\section{Example of Tensor Representations of Joint Distribution}\label{sec_exp}
Consider a single latent variable structure that has three pure observed variables, i.e., $L_1 \to \{X_1, X_2, X_3\}$. We aim to show that the tensor representation of probability contingency table and the tensor rank condition for the joint distribution $\mathbb{P}(X_1X_2X_3)$. For convenience, let $\mathrm{supp}(L_1) = \{0,1\}$ and $\mathrm{supp}(X_i) = \{0,1,2\}$. We further denote $p_{ij} = \mathbb{P}(X_1 =i, X_2 = j)$, $\hat{p_{i|j}} = \mathbb{P}(X_1 =i| L_1 = j)$, $\tilde{p_{i|j}} = \mathbb{P}(X_2 =i| L_1 = j)$, $\bar{p_i} = \mathbb{P}( L_1 = i)$, and $p_{ijk} = \mathbb{P}(X_1 =i, X_2 = j, X_3 = k)$. For the joint distribution of $\mathbb{P}(X_1X_2)$, we have the tensor representation as follows:

\begin{equation}
    \mathcal{T}_{(X_1X_2)} = \begin{bmatrix} p_{00} & p_{01} &p_{02} \\ 
                                           p_{10} & p_{11} &p_{12} \\
                                           p_{20} & p_{21} &p_{22} \end{bmatrix}=
                        \underbrace{
                        \begin{bmatrix} \hat{p_{0|0}} & \hat{p_{0|1}}  \\ 
                                        \hat{p_{1|0}} & \hat{p_{1|1}} \\
                                        \hat{p_{2|0}} & \hat{p_{2|1}} \end{bmatrix}}_{\mathcal{T}_{(X_1|L_1)}}\cdot
                        \underbrace{
                        \begin{bmatrix} \bar{p_{0}} & 0  \\ 
                                        0 & \bar{p_{1}}  \end{bmatrix}}_{\mathrm{Diag}(\mathcal{T}_{(L_1)})}\cdot
                        \underbrace{\begin{bmatrix} \tilde{p_{0|0}} & \tilde{p_{1|0}} &\tilde{p_{2|0}} \\
                                        \tilde{p_{0|1}} & \tilde{p_{1|2}} & \tilde{p_{2|1}} \end{bmatrix}}_{\mathcal{T}^{\intercal}_{(X_2|L_1)}},        
\end{equation}

where $\mathcal{T}_{(X_1|L_1)}$ is the tensor representation of $\mathbb{P}(X_1|L_1)$, $\mathcal{T}_{(X_2|L_1)}$ is the tensor representation of $\mathbb{P}(X_2|L_1)$ and $\mathcal{T}_{(L_1)}$ is the diagonalization of $\mathbb{P}(L_1)$.

Under the Full Rank assumption, we have $\mathcal{T}_{(X_1|L_1)}$ and $\mathcal{T}_{(X_2|L_1)}$ are column full rank. Thus, the rank of $\mathcal{T}_{X_1X_2}$ is two, i.e., $\mathrm{Rank}(\mathcal{T}_{X_1X_2}) = |\mathrm{supp}(L_1)| = 2$. This illustrate the Prop. 1.

Next, we consider the three-way tensor $\mathcal{T}_{(X_1X_2X_3)}$ of the joint distribution $\mathbb{P}(X_1X_2X_3)$. We will represent the three-way tensor as its frontal slices \cite{kolda2009tensor}, i.e., three matrices for $\mathbb{P}(X_1X_2X_3=0)$, $\mathbb{P}(X_1X_2X_3=1)$ and $\mathbb{P}(X_1X_2X_3=2)$.

\begin{equation}
                 \underbrace{\begin{bmatrix} p_{000} & p_{010} &p_{020} \\ 
                                           p_{100} & p_{110} &p_{120} \\
                                           p_{200} & p_{210} &p_{220} \end{bmatrix}}_{\mathcal{T}_{(X_1X_2X_3=0)}},~~~
                \underbrace{\begin{bmatrix} p_{001} & p_{011} &p_{021} \\ 
                                           p_{101} & p_{111} &p_{121} \\
                                           p_{201} & p_{211} &p_{221} \end{bmatrix}}_{\mathcal{T}_{(X_1X_2X_3=1)}},~~~
                \underbrace{\begin{bmatrix} p_{002} & p_{012} &p_{022} \\ 
                                           p_{102} & p_{112} &p_{122} \\
                                           p_{202} & p_{212} &p_{222} \end{bmatrix}}_{\mathcal{T}_{(X_1X_2X_3=2)}}.
\end{equation}

In the contingency tensor above, the element of $\mathcal{T}_{(X_1X_2X_3)}$ is
\begin{equation}
   \begin{aligned}
        &\mathbb{P}(X_1=i,X_2=j, X_3=k)\\
        & = \sum^{1}_{r=0} \mathbb{P}(X_1=i,X_2=j, X_3=k| L_1 = r) \mathbb{P}(L_1=r)\\
        &= \sum^{1}_{r=0} \mathbb{P}(X_1=i|L_1=r)\mathbb{P}(X_2=j|L_1=r)\mathbb{P}(X_3=k|L_1=r)\mathbb{P}(L_1=r).
   \end{aligned}
\end{equation}

One can represent the tensor $\mathcal{T}_{X_1X_2X_3|L_1 = r}$ as

\begin{equation}
    \mathcal{T}_{(X_1X_2X_3|L_1 = r)}=
    \begin{bmatrix} \mathbb{P}(X_1=0|L_1=r) \\ 
                                        \mathbb{P}(X_1=1|L_1=r) \\
                                        \mathbb{P}(X_1=2|L_1=r) \end{bmatrix} \otimes 
       \begin{bmatrix} \mathbb{P}(X_2=0|L_1=r) \\ 
                                        \mathbb{P}(X_2=1|L_1=r) \\
                                        \mathbb{P}(X_2=2|L_1=r) \end{bmatrix} \otimes
       \begin{bmatrix} \mathbb{P}(X_3=0|L_1=r) \\ 
                                        \mathbb{P}(X_3=1|L_1=r) \\
                                        \mathbb{P}(X_3=2|L_1=r) \end{bmatrix},
\end{equation}

where $\otimes$ represent the outer product, e.g., for two vector $\mathbf{u}$ and $\mathbf{v}$, $(\mathbf{u} \otimes \mathbf{v})_{ij} = u_iv_j$ with $\mathbf{u} = \{u_1, \cdots, u_n\}$ and $\mathbf{v} = \{v_1, \cdots, v_n\}$.

According to the definition of tensor rank, one can see that $\mathcal{T}_{(X_1X_2X_3|L_1 = r)}$ is a rank-one tensor. Furthermore, since $\mathcal{T}_{(X_1X_2X_3)} = \sum^{2}_{r=1} \mathcal{T}_{(X_1X_2X_3|L_1 = r)} \mathcal{T}_{(L_1 = r)}$, the rank of $\mathcal{T}_{(X_1X_2X_3)}$ is two under Assumption 2.2 $\sim$ Assumption 2.4. This is because $L_1$ d-separates $X_i$ and $X_j$ with $\forall i,j \in [1,2,3]$, which illistrate the graphical criteria of tensor rank condition.

\section{Discussion of Our Assumptions}

To study the discrete statistical model, certain commonly-used parameter assumptions are necessary. For instance, the Full Rank assumption (Assumption \ref{ass_3}), as utilized in our study, ensures diversity within the parameter space. This is crucial to prevent the contingency table of the joint distribution from collapsing into a lower-dimensional space. There are related works that also use such an assumption \cite{leonard1986bayesian,bartolucci2007extended,gu2022blessing}. Essentially, in our work, such an assumption ensures the effectiveness of rank decomposition (e.g., minimal decomposition or unique decomposition \cite{kruskal1977three}), which induces the structural identifiability of a discrete causal model. Moreover, the sufficient observation assumption that the dimension of latent support is larger than the dimension of observed support is reasonable, as it ensures sufficient measurement of the latent variable. We also discuss this assumption in the remark \ref{pro_test} that demonstrates the CI relation is testable if this assumption holds.

\section{Proofs of Main Results}

\subsection{Proof of Theorem \ref{the:graph}}

\begin{proof}
    "If" part:

    We prove this by contradiction, i.e., suppose (i). there exists a variable set $\mathbf{S}$ in the causal graph $\mathcal{G}$ with $|\mathrm{supp}(\mathbf{S})| = r$ that d-separates any pair of variables in $\{X_1, \cdots, X_n\}$, and (ii).does no exist conditional set $\tilde{\mathbf{S}}$ that satisfies $|\mathrm{supp}(\tilde{\mathbf{S}})| < r$, then $\mathrm{Rank}(\mathcal{T}_{\{X_1 ... X_n\}}) \neq r$. There are two cases we need to consider: Case 1: $\mathrm{Rank}(\mathcal{T}_{\{X_1 ... X_n\}}) > r$, and Case 2: $\mathrm{Rank}(\mathcal{T}_{\{X_1 ... X_n\}}) < r$.

    Case 1: Due to $\mathbf{S}$ is conditional set that d-separates all variables in $\{X_1, \cdots, X_n\}$, then we have $\mathbb{P}(X_1\cdots X_n) = \sum^{r}_{i=1} (\prod^{n}_{j=1} \mathbb{P}(X_j|\mathbf{S}=i)) \mathbb{P}(\mathbf{S}=i)$. When $\mathrm{Rank}(\mathcal{T}_{\{X_1 ... X_n\}}) > r$, let $\mathrm{Rank}(\mathcal{T}_{\{X_1 ... X_n\}}) = k >r$ ,there are

    \begin{equation}
        \begin{aligned}
            &\mathcal{T}_{\{X_1 ... X_n\}}\\
            &= \sum^{r}_{i=1} \mathcal{T}_{(X_1|\mathbf{S}=i)} \otimes \cdots \otimes \mathcal{T}_{(X_n|\mathbf{S}=i)} \mathcal{T}_{(\mathbf{S}=i)}\\
            &=\sum^{k}_{j=1} \mathbf{u}^{(j)}_1 \otimes \cdots \otimes \mathbf{u}^{(j)}_{n},
        \end{aligned}
    \end{equation}
  
    which violates the definition of tensor rank (i.e., $k$ is not a minimal rank-one decomposition, it can be reduced to the smaller decomposition with $r$).
    Meanwhile, if there exists other conditional sets $\tilde{\mathbf{S}}$ such that $\mathbb{P}(X_1\cdots X_n) = \sum^{k}_{i=1} (\prod^{n}_{j=1}\mathcal{T}_{(X_j|\tilde{\mathbf{S}}=i)}) \mathcal{T}_{(\tilde{\mathbf{S}}=i)}$ is minimal rank decomposition, it violates the condition (ii) that $|\mathrm{supp}(\mathbf{S})|$ is minimal.
,

    Case 2: When $\mathrm{Rank}(\mathcal{T}_{\{X_1 ... X_n\}}) < r$, let $\mathrm{Rank}(\mathcal{T}_{\{X_1 ... X_n\}}) = t <r$, due to $\mathbf{S}$ is conditional set with smallest support $r$ in the causal graph, one have 

    \begin{equation}
        \begin{aligned}
            &\mathcal{T}_{\{X_1 ... X_n\}}\\
            &= \sum^{r}_{i=1} \mathcal{T}_{(X_1|\mathbf{S}=i)} \otimes \cdots \otimes \mathcal{T}_{(X_n|\mathbf{S}=i)} \mathcal{T}_{(\mathbf{S}=i)}\\
            &=\sum^{t}_{j=1} \mathbf{u}^{(j)}_1 \otimes \cdots \otimes \mathbf{u}^{(j)}_{n},
        \end{aligned}
    \end{equation}

    which means that there at least exist $\mathcal{T}_{(X_1 ... X_n| \mathbf{S}=i)}$ and $\mathcal{T}_{(X_1 ... X_n| \mathbf{S}=j)}$ such that $\mathcal{T}_{(X_j| \mathbf{S}=i)} = \alpha \mathcal{T}_{(X_k| \mathbf{S}=i)}$ for any $i, k \in [n]$ where $\alpha$ is a constant, i.e, the columns of the conditional contingency table are linearly dependent.  This violates the assumption 2.4 (the Full Rank assumption).

    Therefore, $\mathrm{Rank}(\mathcal{T}_{\{X_1 ... X_n\}}) = r$ if the condition (i) and condition (ii) holds.

    "Only if" part:

    We will show if one of the conditions is violated, the tensor rank is not $r$, i.e.,  if (i). there does not exist a variable set $\mathbf{S}$ in the causal graph with $|\mathrm{supp}(\mathbf{L})| = r$ that d-separates any pair of variables in $\{X_1, \cdots, X_n\}$, or (ii). exist $\tilde{\mathbf{S}}$ that satisfies $|\mathrm{supp}(\tilde{\mathbf{S}})| < r$, then $\mathrm{Rank}(\mathcal{T}_{\{X_1 ... X_n\}}) \neq r$. 

    We first show the case that condition (i) is violated. There are two cases we need to consider, i.e., Case 1: $\mathbf{S}$ is not a conditional set in the causal graph, and Case 2: $\mathbf{S}$ is a variable constructed from parameter space.

    Case 1: if $\mathbf{S}$ is not a conditional set and $\mathbf{S} \cap \mathrm{Des}_{X_1} \cap \cdots \cap \mathrm{Des}_{X_n} = \emptyset$, by the Markov assumption, $\mathbb{P}(X_1, \cdots, X_n | \mathbf{S}=i) \neq \prod^{n}_{j=1} \mathbb{P}(X_j | \mathbf{S}=i)$. By Lemma. \ref{lemma_onlyif}, $\mathbb{P}(X_1, \cdots, X_n | \mathbf{S}= i)$ is not a rank-one tensor, which violates the definition of tensor rank.

    If $\mathbf{S}$ is not a conditional set and $\mathbf{S} \cap \mathrm{Des}_{X_1} \cap \cdots \cap \mathrm{Des}_{X_n} \neq \emptyset$, we will show that $\mathbb{P}(X_1, \cdots, X_n| \mathbf{S}=i)$ is not a rank-one tensor, to prove such a rank-decomposition does not exist. Under the faithfulness assumption and Markov assumption, let $\tilde{\mathbf{S}}$ be the a minimal conditional set with $|\mathrm{supp}(\tilde{\mathbf{S}})| = k$ for any pair variable in $\{X_1, \cdots, X_n\}$, we have

    \begin{equation}
        \begin{aligned}
            &\mathcal{T}_{(X_1 \cdots X_n|\mathbf{S}=i)}\\
            &=\sum^{k}_{j=1} \mathcal{T}_{(X_1 \cdots X_n|\mathbf{S}=i,\tilde{\mathbf{S}}=j)}\mathcal{T}_{(\mathbf{S}=i|\tilde{\mathbf{S}}=j)}.
            % \mathbb{P}(X_1 \cdots X_n|\mathbf{S}=i, \tilde{\mathbf{S}}=j)\mathbb{P}(\mathbf{S}=i | \tilde{\mathbf{S}}=j).
        \end{aligned}
    \end{equation}

    If $\mathcal{T}_{(X_1 \cdots X_n|\mathbf{S}=i)}$ is a rank-one tensor, i.e., $\mathcal{T}_{(X_1 \cdots X_n|\mathbf{S}=i)} = \mathbf{u}_1 \otimes \cdots \otimes \mathbf{u}_n$, we further have

    \begin{equation}\label{eq1}
    \sum^{k}_{j=1} \mathcal{T}_{(X_1 \cdots X_n|\mathbf{S}=i,\tilde{\mathbf{S}}=j)}\mathcal{T}_{(\mathbf{S}=i|\tilde{\mathbf{S}}=j)} = \mathbf{u}_1 \otimes \cdots \otimes \mathbf{u}_n.
    \end{equation}

    Note that if there exists $X_p, X_q \in \{X_1, \cdots, X_n\}$ such that $\mathbf{S}$ is the common descendant variable of $X_p$ and $X_q$, it will lead to a collider structure in which $\mathbb{P}(X_i| \mathbf{S})$ and $\mathbb{P}(X_j| \mathbf{S})$ is relevant (i.e., the v-structure is activated).  Thus, let $\mathbf{X}_t = \{X_1, \cdots, X_n\} \setminus \{X_i, X_j\}$, $\mathcal{T}_{(X_1 \cdots X_n|\mathbf{S}=i)}$ is not a rank-one tensor due to the sub-tensor $\mathcal{T}_{(X_iX_j, \mathbf{X}_t = \mathbf{c}|\mathbf{S}=i)}$ (a slice of tensor $\mathcal{T}_{(X_1 \cdots X_n|\mathbf{S}=i)}$) is not a rank-one tensor \footnote{The lower bound of tensor rank is not less than the rank of any slice of tensor.}\cite{hackbusch2012tensor,kruskal1977three}, under the faithfulness assumption. If so, one have $\mathcal{T}_{(X_iX_j, \mathbf{X}_t = \mathbf{c}|\mathbf{S}=i)} = \mathbf{u}_1 \otimes \mathbf{u}_2$. Let $\mathbf{u}_1 = \mathbb{P}(X_i, \mathbf{X}_t = \mathbf{c}| \mathbf{S} = i)$ and $\mathbf{u}_2 = \mathbb{P}(X_j, \mathbf{X}_t = \mathbf{c}| \mathbf{S} = i)$, it violates the faithfulness assumption.

    Thus, $\mathbf{S}$ can not be the common descendant of any pair variables in $\{X_1, \cdots, X_n\}$. So, for any $X_p, X_q \in \{X_1, \cdots, X_n\}$, there are $\mathcal{T}_{(X_pX_q|\mathbf{S}=j \tilde{\mathbf{S}}=i)} = \mathcal{T}_{(X_p|\tilde{\mathbf{S}}=i, \mathrm{Des}_{X_p})=c} \otimes \mathcal{T}_{(X_q|\tilde{\mathbf{S}}=i, \mathrm{Des}_{X_q})=c}$ according to Markov assumption ($\tilde{\mathbf{S}}$ is a d-separation set for $X_p$ and $X_q$).

    Now, for the Eq. \ref{eq1}, we have 

    \begin{equation}
    \begin{aligned}
        &\mathcal{T}_{(X_1 \cdots X_n|\mathbf{L}=i)}\\
        & = \sum^{k}_{j=1} \otimes^{n}_{t=1}\mathcal{T}_{(X_t|\tilde{\mathbf{S}}=j, \mathrm{Des}_{X_t}=c)} \mathcal{T}_{(\tilde{\mathbf{S}}=j, \mathrm{Des}_{X_t}=c|\tilde{\mathbf{S}}=j)}\\
        &=\mathbf{u}_1 \otimes \cdots \otimes \mathbf{u}_n.
    \end{aligned}
    \end{equation}

    This equality holds if the $k$ sum of the rank-one tensor can be reduced to a rank-one tensor, i.e., for any $X_p$, $\mathcal{T}_{(X_p|\tilde{\mathbf{S}}=1, \mathrm{Des}_{X_p} = c)}=\alpha_{2} \mathcal{T}_{(X_p|\tilde{\mathbf{S}}=2, \mathrm{Des}_{X_p} = c)}= \cdots= \alpha_{r} \mathcal{T}_{(X_p|\tilde{\mathbf{S}}=r, \mathrm{Des}_{X_p} = c)}$, where $\alpha_{i}$ and $c$ are constant. However, this equality can not hold due to the following reasons.
    
    If $\mathrm{Pa}_{X_p} \neq \emptyset$, let $L_p$ be the parent of $X_p$,  we have 

    \begin{equation}
        \begin{aligned}
            &\mathcal{T}_{(X_p|\tilde{\mathbf{S}}=i, \mathrm{Des}_{X_p} = c)}\\
            &=\sum^{|\mathrm{supp}(L_p)|}_{j=1}\mathcal{T}_{(X_p|L_p=j, \mathrm{Des}_{X_p}=c)}\mathcal{T}_{(L_p=j|\tilde{\mathbf{S}}=i, \mathrm{Des}_{X_p} = c)}\\
            &\neq \sum^{|\mathrm{supp}(L_p)|}_{j=1}\mathcal{T}_{(X_p|L_p=j, \mathrm{Des}_{X_p}=c)}\mathcal{T}_{(L_p=j|\tilde{\mathbf{S}}=r, \mathrm{Des}_{X_p} = c)} = \alpha \mathcal{T}_{(X_p|\tilde{\mathbf{S}}=r, \mathrm{Des}_{X_p} = c)},
        \end{aligned}
    \end{equation}

    in which the inequality holds because of the Full Rank assumption and all marginal distribution probabilities are not zero (see Lemma. \ref{lemma_fullrank}). Therefore, $\mathcal{T}_{(X_p|\tilde{\mathbf{S}}=1, \mathrm{Des}_{X_p} = c)} \neq \cdots \neq \alpha_{r} \mathcal{T}_{(X_p|\tilde{\mathbf{S}}=r, \mathrm{Des}_{X_p} = c)}$. Thus, $\mathcal{T}_{(X_1 \cdots X_n|\mathbf{S}=i)}$ is not a rank-one tensor. By the definition of tensor rank, $\mathrm{Rank}(\mathcal{T}_{(X_1 \cdots X_n)}) \neq r$.

    If $\mathrm{Pa}_{X_p} = \emptyset$, i.e., $X_p$ is the root variable, (e.g., $X_p \to \mathbf{S}$, $\mathrm{Des}_{X_p \in \mathbf{S}}$ ), we will show that the probability contingency table $\mathcal{T}_{(X_p| \mathbf{S})}$ is full rank, and then the equality in Eq. \ref{eq1} can not hold. According to the Full Rank assumption and $X_p$ is the parent variable of $\mathbf{S}$, we have the probability contingency table $\mathcal{T}_{(\mathbf{S}|X_p)}$ is full rank.

    Due to $\mathbb{P}(X_p, \mathbf{S}) = \mathbb{P}(\mathbf{S}|X_p) \mathbb{P}(X_p) = \mathbb{P}(X_p|\mathbf{S}) \mathbb{P}(\mathbf{S})$ and the probability in the marginal distribution is not zero, we have $\mathcal{T}_{(X_p|\mathbf{S})} = \mathcal{T}_{(\mathbf{S}|X_p)} \mathrm{Diag}(\mathcal{T}_{(X_p)}) \mathrm{Diag}(\mathcal{T}_{(\mathbf{S})})^\dagger$, where $\mathrm{Diag}(\mathcal{T}_{(X_p)})$ is a diagonalization of marginal distribution probability vector of $X_p$. One can see that $\mathcal{T}_{(X_p|\mathbf{S})}$ is full rank due to the diagonal matrices are all of full rank. Thus, $\mathcal{T}_{(X_p|\tilde{\mathbf{S}}=1)} \neq \cdots \neq \alpha_{r} \mathcal{T}_{(X_p|\tilde{\mathbf{S}}=r)}$ and $\mathcal{T}_{(X_1 \cdots X_n|\mathbf{S}=i)}$ is not a rank-one tensor. By the definition of tensor rank, $\mathrm{Rank}(\mathcal{T}_{(X_1 \cdots X_n)}) \neq r$.

    Case 2: we further show that for the parameter space, the tensor $\mathcal{T}_{(X_1, \cdots, X_n)} = \sum^{r}_{i=1}\mathbf{u}_{1} \otimes \cdots \otimes \mathbf{u}_{n}$ does not hold with $r$, where $\mathbf{u}_{i}$ represents any vector. Let $\mathbf{u}_{i} \otimes \cdots \otimes \mathbf{u}_{n}$ be a rank-one tensor of $\mathbb{P}(X_1, \cdots, X_n| \tilde{\mathbf{S}}=i) \mathbb{P}(\tilde{\mathbf{S}}=i)$ due to $\mathbb{P}(\tilde{\mathbf{S}}=i)$ is a constant. In other words, one can construct a variable set $\tilde{\mathbf{S}}$ with $|\mathrm{supp}(\tilde{\mathbf{S}})| = r$. For any $\tilde{\mathbf{S}}$ constructed from parameter space (i.e., $\tilde{\mathbf{S}}$ is not a true node set in the causal graph), if $\mathcal{T}_{(X_1, \cdots, X_n| \tilde{\mathbf{S}}=i)} = \mathbf{u}_{1} \otimes \cdots \otimes \mathbf{u}_{n}$, one can let $\mathbf{u}_1 = \mathcal{T}_{(X_1|\tilde{\mathbf{S}}=i)}, \cdots, \mathbf{u}_n = \mathcal{T}_{(X_n|\tilde{\mathbf{S}}=i)}$, which violates the faithfulness assumption. Based on the above analysis, there does not exist $\tilde{\mathbf{S}}$ by any constructed such that $\mathcal{T}_{(X_1\cdots X_n)}$ have the summation $r$ rank-one decomposition, i.e., $\mathrm{Rank}(\mathcal{T}_{(X_1 \cdots X_n)}) \neq r$.

    Now, we analyze the condition (ii), i.e., there exists a conditional set $\tilde{\mathbf{S}}$ with $\mathrm{supp}(\tilde{\mathbf{S}}) < r$. Let $\mathrm{supp}(\tilde{\mathbf{S}}) = k$, $k < r$, we have $\mathcal{T}_{(X_1 \cdots X_n)} = \sum^{k}_{i=i} \mathcal{T}_{(X_1|\tilde{\mathbf{S}} = i)} \cdots \otimes \mathcal{T}_{(X_n|\tilde{\mathbf{S}} = i) }\mathcal{T}_{(\tilde{\mathbf{S}} = i)}$ is a smaller rank-one decomposition than $\mathcal{T}_{(X_1 \cdots X_n)} = \sum^{r}_{j=1} \mathbf{u^{j}}_1 \otimes \cdots \otimes \mathbf{u}^{j}_n$. According to the definition of tensor rank, we have $\mathrm{Rank}(\mathcal{T}_{\{X_1 ... X_n\}}) =k \neq r$.

    In summary, the theorem is proven.

\end{proof}

\begin{lemma}\label{Lemma_independent}
    Let $L_p$ is the parent of $X_p$, for $\mathbb{P}(X_p|\mathbf{S}=i)$ and $\mathbb{P}(X_p|\mathbf{S}=r)$, one have $\sum^{|\mathrm{supp}(L_p)|}_{j=1}\mathcal{T}_{(X_p|L_p=j)}\mathcal{T}_{(L_p=j|\mathbf{S}=i)} \neq \sum^{|\mathrm{supp}(L_p)|}_{j=1}\mathcal{T}_{(X_p|L_p=j)}\mathcal{T}_{(L_p=j|\mathbf{S}=r)}$ hold under the Full Rank  assumption.
\end{lemma}

\begin{proof}
    We prove it by contradiction. For convenience of symbols, let $\mathbf{v}_j = \mathcal{T}_{(X_p|L_p=j)}$, and $\alpha_{j|i} = \mathbb{P}(L_p=j|\mathbf{L}=i)$ and $\beta_{j|r} = \mathbb{P}(L_p=j|\mathbf{L}=r)$ ($\alpha_{j|i} \neq \beta_{j|r}$), if the equality hold, we have

    \begin{equation}
        \sum^{|\mathrm{supp}(L_p)|}_{j=1} \alpha_{j|i} \mathbf{v}_j = \sum^{|\mathrm{supp}(L_p)|}_{j=1} \beta_{j|r} \mathbf{v}_j,
    \end{equation}
\end{proof}

which means that

\begin{equation}
    \mathbf{v}_t = \frac{\sum^{|\mathrm{supp}(L_p)|}_{k \neq t}(\alpha_{k|i} - \beta_{k|r})\mathbf{v}_k}{\alpha_{t|i} - \beta_{t|r}}.
\end{equation}

That is, $\mathbf{v}_t$ is a linear combination of other vectors $\mathbf{v}_k$ with $t \neq k$, i.e., the linear combination of other column vectors in the conditional probability contingency table, which is contrary to the Full Rank assumption.

\begin{lemma}\label{lemma3}
     Let $L_p$ be the parent of $X_p$, suppose the assumption 2.1 $\sim$ assumption 2.3 hold. $\mathcal{T}_{(X_p|\mathbf{S}=i)}$ can not be a linear combination of other $q$ vectors $\mathcal{T}_{(X_p|\mathbf{S}=r)}$, i.e., $\mathcal{T}_{(X_p|\mathbf{S}=i)} \neq \sum^{q}_{r=1} \gamma_r \mathcal{T}_{(X_p|\mathbf{S}=r)}$.
\end{lemma}

\begin{proof}
    We prove it by contradiction.
    For convenience of symbols, let $\mathbf{v}_j = \mathcal{T}_{(X_p|L_p=j)}$, and $\alpha_{j|i} = \mathbb{P}(L_p=j|\mathbf{S}=i)$ and $\beta_{j|r} = \mathbb{P}(L_p=j|\mathbf{S}=r)$, if the equality hold, one have

    \begin{equation}
    \begin{aligned}
        &\sum^{|\mathrm{supp}(L_p)|}_{j=1} \alpha_{j|i} \mathbf{v}_j = \sum^{q}_{r=1}\gamma_{r} \sum^{|\mathrm{supp}(L_p)|}_{j=1} \beta_{j|r} \mathbf{v}_j\\
        &= \sum^{|\mathrm{supp}(L_p)|}_{j=1} \mathbf{v}_j \sum^{q}_{r=1}\gamma_{r} \beta_{j|r},
    \end{aligned}
    \end{equation}

there exist a vector $\mathbf{v}_t$ with $\alpha_{t|i} - \sum^{q}_{r=1}\gamma_{r}\beta_{t|r} \neq 0$, such that

\begin{equation}
    \mathbf{v}_t = \frac{\sum^{|\mathrm{supp}(L_p)|}_{k \neq t}(\sum^{q}_{r=1}\gamma_{r} \beta_{j|r} - \alpha_{k|i}) \mathbf{v}_k}{\alpha_{t|i} - \sum^{q}_{r=1}\gamma_{r}\beta_{t|r}}.
\end{equation}

It means $\mathcal{T}_{(X_p|L_p = t)}$ is a linear combination of other column vectors in the conditional probability contingency table, which is contrary to the Full Rank assumption.

\end{proof}

\begin{lemma}\label{lemma_onlyif}
    For $\{X_1, \cdots, X_n\}$, if $\tilde{\mathbf{S}}$ is not a conditional set that d-separates any pair variables in $\{X_1, \cdots, X_n\}$ in the causal graph and $\tilde{\mathbf{S}} \cap \mathrm{Des}_{\{X_1 \cdots X_n\}} = \emptyset$, then $\mathcal{T}_{(X_1 \cdots X_n|\tilde{\mathbf{S}} = j)}$ is not a rank-one tensor.
\end{lemma}

\begin{proof}
    Let $\mathbf{S}$ be the minimal conditional set, (e.g., $\mathbf{S} = \{X_1, \cdots, X_{n-1}\}$), denote $|\mathrm{supp}(\mathbf{S})| = r$, under the faithfulness assumption and the Markov assumption, $\mathbb{P}(\mathbf{S}) \neq \mathbb{P}(\tilde{\mathbf{S}})$, if $\mathcal{T}_{(X_1\cdots X_n|\tilde{\mathbf{S}}=j)}$ is a rank-one tensor, we have

    \begin{equation}
        \begin{aligned}
            &\mathcal{T}_{(X_1\cdots X_n|\tilde{\mathbf{S}}=j)}\\
            &=\sum^{r}_{i=1} \otimes^{n}_{t=1} \mathcal{T}_{(X_t|\mathbf{S} = i)} \mathcal{T}_{(\mathbf{S} = i|\tilde{\mathbf{S}} = j)}\\
            &=\mathbf{u}_1 \otimes \cdots \otimes \mathbf{u}_n,
        \end{aligned}
    \end{equation}
which means that for any $X_p \in \{X_1, \cdots X_n\}$, $\mathcal{T}_{(X_p|\mathbf{S}=1)} = \alpha_{2} \mathcal{T}_{(X_p|\mathbf{S}=2)} = \cdots \mathcal{T}_{(X_p|\mathbf{S}=r)}$. 

If $X_p$ has a parent variable $L_p$ in the causal graph, by Lemma .\ref{Lemma_independent}, the equality does not hold. Thus, $\mathcal{T}_{(X_1\cdots X_n|\tilde{\mathbf{L}}=j)}$ is not a rank-one tensor.

If $X_p$ is root node in the causal graph, and $\mathbf{S}$ is conditional set that d-separates $X_p$ and $\{X_1, \cdots X_n\} \setminus \{X_p\}$, we have $\mathcal{T}_{(X_p|\mathbf{S})}$ is full rank. The reason is the following.

Since $\mathbb{P}(X_p|\mathbf{S}) \mathbb{P}(\mathbf{S}) = \mathbb{P}(\mathbf{S}|X_p) \mathbb{P}(X_p)$ by Bayes' theorem \cite{koch1990bayes}, and due to all marginal probabilities are not zero, then we have $\mathcal{T}_{(X_p|\mathbf{S})} = \mathcal{T}_{(\mathbf{S}|X_p)} \mathrm{Diag}(\mathcal{T}_{(X_p)}) \mathrm{Diag}(\mathcal{T}_{(\mathbf{S})})^\dagger$, where $\dagger$ is inverse of matrix, and $X_p$ is ancestor of $\mathbf{S}$. By Lemma .\ref{lemma3}, $\mathbf{S}$ has a parent variables and hence $\mathcal{T}_{(\mathbf{S}|X_p)}$ is full rank (one can vectorize the variable set $\mathbf{S}$ as a variable). Now, we have $\mathcal{T}-{(X_p|\mathbf{S})}$ is full rank due to the three matrices on the right are all full rank.

Thus, $\mathcal{T}_{(X_p|\mathbf{S}=1)} = \alpha_{2} \mathcal{T}_{(X_p|\mathbf{S}=2)} = \cdots \mathcal{T}_{(X_p|\mathbf{S}=r)}$ does not hold, i.e., $\mathcal{T}_{(X_1\cdots X_n|\tilde{\mathbf{S}}=j)}$ is not a rank-one tensor.

\end{proof}

\begin{lemma}\label{lemma_fullrank}
    In the discrete causal graph, for the variable $X_p$, let $L_p$ be the vectorization of parent set of $X_p$, and $\mathrm{Des}_{X_p}$ is descendant variable of $X_p$, then $\mathcal{T}_{(X_p|L_p, \mathrm{Des}_{X_p} = j)}$ is full rank.
\end{lemma}

\begin{proof}
    By Bayes' theorem \cite{koch1990bayes}, we have 

    \begin{equation}
        \begin{aligned}
            &\mathcal{T}_{(X_p, L_p |\mathrm{Des}_{X_p} = j)}\\
            &= \mathcal{T}_{(X_p|L_p, \mathrm{Des}_{X_p} = j)} \mathrm{Diag}(\mathcal{T}_{(L_p| Des_{X_p} = j)}) \\
            &= \mathcal{T}_{(L_p|X_p, \mathrm{Des}_{X_p} = j)} \mathrm{Diag}(\mathcal{T}_{(X_p| \mathrm{Des}_{X_p} = j)}). 
        \end{aligned}
    \end{equation}

    Since $X_p$ and $\mathrm{Des}_{X_p}$ both are the descendant set of $L_p$, we have 

    \begin{equation}
        \mathcal{T}_{(L_p|\mathrm{Des}_{L_p})} = \mathcal{T}_{(\mathrm{Des}_{L_p}|L_p)}\mathrm{Diag(\mathcal{T}_{(L_p)}}\mathrm{Diag(\mathcal{T}_{(\mathrm{Des}_{L_p})}}^\dagger,
    \end{equation}
    due to all marginal distribution probabilities are not zero. Then $\mathcal{T}_{(L_p|\mathrm{Des}_{L_p})}$ is full rank, i.e., $\mathcal{T}_{(L_p|X_p, \mathrm{Des}_{X_p} = j)}$ also full rank.

    Moreover, for any $\mathcal{T}_{(L_p=i | \mathrm{Des} = j)}$, we have

    \begin{equation}
        \mathcal{T}_{(L_p = i| \mathrm{Des}_{X_p} = j)} = \frac{\mathcal{T}_{(L_p = i, \mathrm{Des}_{X_p} = j)}}{\mathcal{T}_{(\mathrm{Des}_{X_p} = j)}} \neq 0,
    \end{equation}
    because all marginal distribution probabilities are not zero. Thus, we have 

        \begin{equation}
        \begin{aligned}
            &\mathcal{T}_{(X_p, L_p |\mathrm{Des}_{X_p} = j)}\\
            &= \mathcal{T}_{(L_p|X_p, \mathrm{Des}_{X_p} = j)} \mathrm{Diag}(\mathcal{T}_{(X_p| \mathrm{Des}_{X_p} = j)}) \mathrm{Diag}(\mathcal{T}_{(L_p| \mathrm{Des}_{X_p} = j)})^\dagger. 
        \end{aligned}
    \end{equation}

    One can see that $\mathcal{T}_{(X_p, L_p |\mathrm{Des}_{X_p} = j)}$ is full rank due to three matrices on the right side being full rank.

\end{proof}

\subsection{Proof of Proposition \ref{Pro_rank}}

\begin{proof}
The proof is straightforward. In the discrete LSM model, suppose all latent variable has the same state space. Any two observed are d-separated by any one of their latent parents. According to the graphical criteria, the rank of tensor $\mathcal{T}_{(X_iX_j)}$ is the dimension of latent support.

\end{proof}

\subsection{Proof of Proposition \ref{prop_cluster}}

\begin{proof}

\textbf{Proof of $\mathcal{R}$ule 1:}
In the discrete latent variable model and suppose the assumption 2.1 $\sim$ assumption 2.3 holds, if there does not exists a latent variable $L_1$ that d-separates any pair variables in $\{X_i, X_j, X_k\}$, i.e., the rank of tensor $\mathcal{T}_{(X_iX_jX_k)}$ is not $r$ (by Theorem 1), it must be the full-connection structure among latent variables, as shown in Fig. \ref{fig:c2n3} (d). Otherwise, one can find one latent variable $L_1$ that can d-separates $\{X_i, X_j, X_k\}$, as shwon in Fig. \ref{fig:c2n3} (a) $\sim$ (c). We will show that, if only consider one of the latent variables of $\{L_1, L_2, L_3\}$ in Fig. \ref{fig:c2n3} (d), the tensor of $\mathcal{T}_{(X_iX_jX_k)}$ can not have rank-one decomposition.

\begin{figure}[htp]
    \centering
    \includegraphics[width=0.95\textwidth]{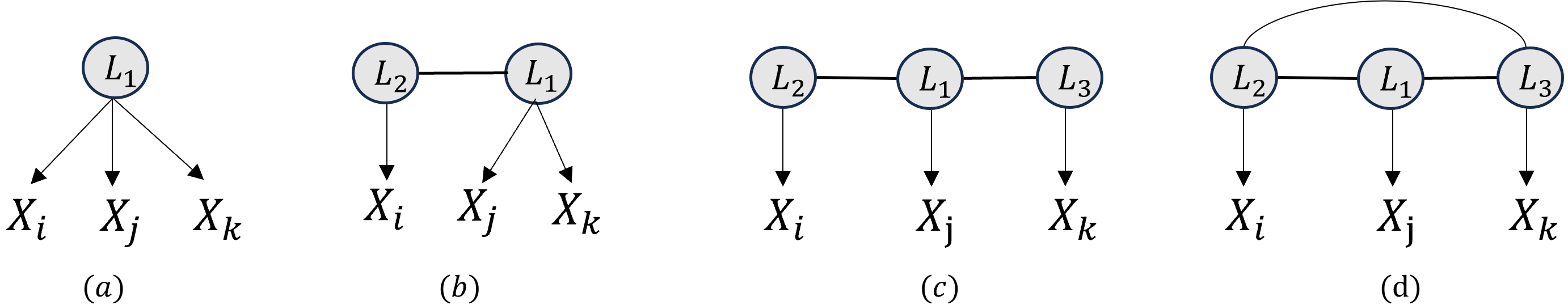}
    \caption{Illustrative example for $\mathcal{R}$ule 1.}
    \label{fig:c2n3}
\end{figure}

According to the graphical criteria of tensor rank and $r < d_i$, and suppose all latent variables have the same state space, for $X_i$, $X_j$ and $X_k$, there is not only one latent variable is conditional set, i.e., but there also is not one latent variable $L$ that d-separates any pair variables in $\{X_i, X_j, X_k\}$. Thus, $\mathrm{Rank}(\mathcal{T}_{(X_iX_jX_k)}) \neq r$.

\textbf{Proof of $\mathcal{R}$ule 2:}

We aim to show if Rank$(\mathcal{T}_{(X_iX_jX_kX_s)})=r$ for any $X_s \in \mathbf{X} \setminus \{X_i, X_j, X_k\}$ then there exists a latent variable $L_p$ that d-separates $\{X_i, ..., X_s\}$ and $L_p$ is the parent variable of $\{X_i, X_j, X_k\}$ in the discrete LVM. We first prove it by the contradiction. If $\{X_i, X_j, X_k\}$ does not share one common latent parent, e.g., $L_1 \to \{X_i, X_j\}$ and $L_2 \to \{X_k\}$, due to the structure assumption in discrete LSM, there exist $X_s \in Ch_{L_2}$, as shown in Fig. \ref{fig:c2nb3}.

\begin{figure}[htp]
    \centering
    \includegraphics[width=0.35\textwidth]{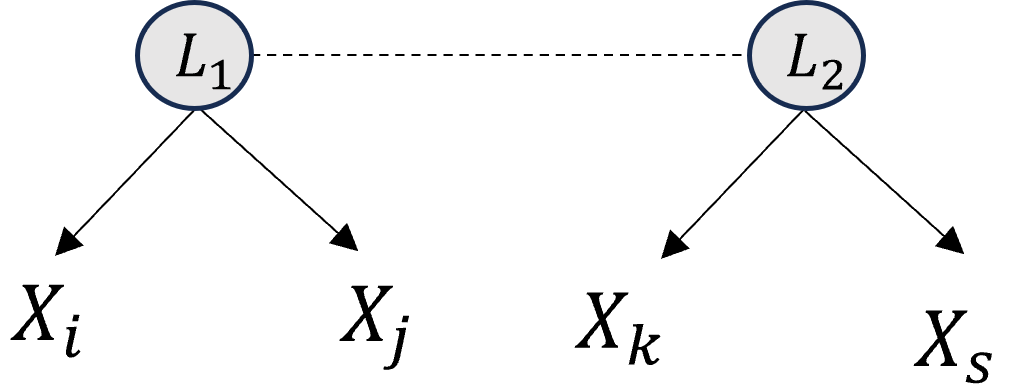}
    \caption{Illustrative example for $\mathcal{R}$ule 2.}
    \label{fig:c2nb3}
\end{figure}

By the graphical criteria of tensor rank condition and $r < d_i$, one has $\mathrm{Rank}(\mathcal{T}_{(X_iX_jX_kX_s)}) \neq r$ due to $L_1$ or $L_2$ is not the conditional set that d-separates any pair variable of $\{X_i, X_j, X_k, X_s\}$ (e.g., $X_k, X_s$ can not be d-separates given $L_1$), which is contrary to the condition $\mathrm{Rank}(\mathcal{T}_{(X_iX_jX_kX_s)}) = r$.

\end{proof}

\subsection{Proof of Proposition \ref{pro_merger}}

\begin{proof}
    Since $C_1$ and $C_2$ are two causal clusters, then the elements in $C_1$ have only one common latent variable. Without loss of generality, we let $L_1$ denote the parental latent variable of $C_1$. Similarly, $L_2$ denotes the parental latent variable of $C_2$. Since $C_1$ and $C_2$ are overlapping, then they have at least one shared element. Let $X_i$ denote the shared element of $C_1$. then $X_i$ has two latent parents $L_1$ and $L_2$, which contradicts with the pure child assumption in the discrete LSM model. This finishes the proof.
\end{proof}

\subsection{Proof of Theorem \ref{the:mm}}

\begin{proof}
   Based on Prop. 2, one can identify the causal cluster by testing the tensor rank condition (Line 5 $\sim$ 12).  Besides, the Prop. 3 ensure that there are no redundant latent variables introduced (Line 15). Thus, the causal cluster can be identified by Algorithm 1, under the discrete latent variable model, with assumption 2.2 $\sim$ assumption 2.4.
\end{proof}

\subsection{Proof of Theorem \ref{the:d-separation}}

\begin{proof}

We first prove this result by a specific case and then extend it to a general case result.

\textbf{Proof by Specific case}

\begin{figure}[htp]
    \centering
    \includegraphics[width=0.20\textwidth]{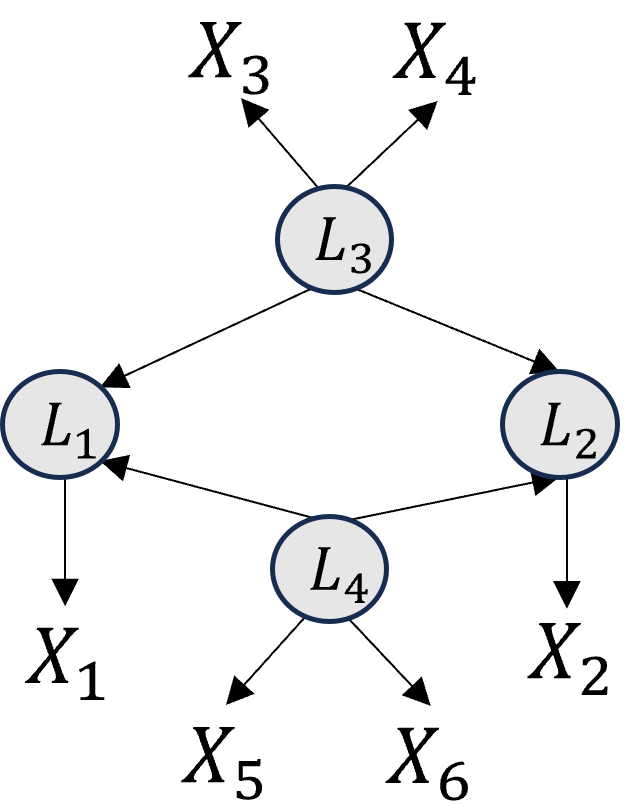}~~~~~~~~~~
    \includegraphics[width=0.20\textwidth]{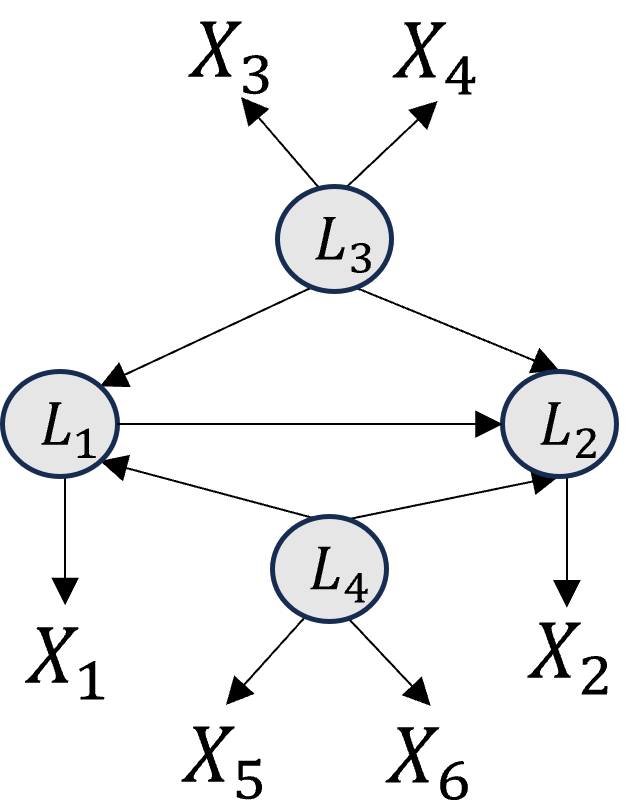}
    \caption{d-separation and d-connection example.}
    \label{fig:dseparation}
\end{figure}

'If' part: as shown in Fig.\ref{fig:dseparation}, suppose all support of latent space is $r$ and $r < d_i$, $d_i$ is the dimension of any support of observed variables. We will show that, for $\{X_1,..., X_6\}$, the rank of $\mathcal{T}_{\{X_1,..., X_6\}}$ are $r^2$. Let $\mathbf{L}$ be the vectorization of joint distribution $\mathbb{P}(L_3, L_4)$ with $|\mathrm{supp}(\mathbf{L})| = r^2$. 

Since $L_2, L_3$ is a conditional set that d-separates all variables in $\{X_1, \cdots, X_6\}$ and there is no other conditional set with smaller support, according to the graphical criteria of tensor rank, $\mathrm{Rank}(\mathcal{T}_{(X_1\cdots X_6)}) = |\mathrm{supp}(\mathbf{L})| = r^2$.

'Only if' part: now, we consider the case that if $L_1$ and $L_2$ are d-connection (Fig.\ref{fig:dseparation}). In this case, given $\mathbb{P}(L_3L_4)$ (represented by $\mathbb{P}(\mathbf{L})$), $\mathcal{T}_{(X_1X_2|\mathbf{L} = i)}$ cannot be decomposition as the outer product of two vectors due to the contingency table of $\mathcal{T}_{(X_1X_2|\mathbf{L} = i)}$ is not a rank-one tensor by Markov assumption. That is, $\mathbf{L}$ is not a conditional set that d-separates $X_1$ and $X_2$.

According to the graphical criteria of tensor rank condition, one can infer that $\mathrm{Rank}(\mathcal{T}_{(X_1 \cdots X_6)}) \neq |\mathrm{supp}(\mathbf{L})| = r^2$. Based on the above analysis, one can see that the conditional independent relations hold if and only if the rank of the tensor equals to support dimension of the conditional set.

\textbf{Proof by general case}

'if' part:
in the discrete LSM model, assume that all latent variables have the same support. For any pair of variables $X_i$ and $X_j$ that are direct children of $L_k$, it is evident that $L_k$ is the only minimal conditional set that d-separates $X_i$ from $X_j$. Consider a set of latent variables $\mathbf{L}_p$ and their corresponding child sets $\mathbf{X}_{p1}$ and $\mathbf{X}_{p2}$, where each latent variable has at least two child variables included in $\mathbf{X}_{p1}$ and $\mathbf{X}_{p2}$. The minimal conditional set that d-separates all variables in $\mathbf{X}_{p1}$ from those in $\mathbf{X}_{p2}$ is their common latent parent set $\mathbf{L}_p$. Moreover, for two variables $X_i$ and $X_j$ that do not share a common latent parent, if $\mathbf{L}_p$ d-separates $X_i$ and $X_j$, then $\mathbf{L}_p$ also d-separates all variables in $\mathbf{X}_{p1} \cup \mathbf{X}_{p2} \cup \{X_i, X_j\}$. Since all latent variables have the same support, the minimal conditional set in the graph corresponds to the set with the minimal support. Therefore, by the graphical criteria of tensor rank, $\mathrm{Rank}(\mathcal{T}{(X_i, X_j, \mathbf{X}_{p1}, \mathbf{X}_{p2})}) = |\mathrm{supp}(\mathbf{L}_p)|$. As all latent variables share the same state space, we deduce that $|\mathrm{supp}(\mathbf{L}_p)| = r^{|\mathbf{L}_p|}$.

'Only if' part:
if $X_i$ and $X_j$ cannot be d-separated by $\mathbf{L}_p$, then $\mathbf{L}_p$ does not constitute a conditional set for $\mathbf{X}_{p1} \cup \mathbf{X}_{p2} \cup \{X_i, X_j\}$. According to the graphical criteria of tensor rank, $\mathrm{Rank}(\mathbb{P}(X_i, X_j, \mathbf{X}_{p1}, \mathbf{X}_{p2})) \neq |\mathrm{supp}(\mathbf{L}_p)|$.

\end{proof}

\subsection{Proof of Theorem \ref{the_stru}}

\begin{proof}
    Such an identification is derived from the original PC algorithm. By Theorem 2, given the measurement model, one can test the CI relations among latent variables, when $r < d_i$ (remark 1). Thus, the causal structure among latent variables can be identified up to a Markov equivalent class by Algorithm 2 \cite{spirtes2000causation}. 
\end{proof}

\section{Extension of Different Latent State Space}\label{sec_dif}

To extend our theoretical result to the case in which the state space of the latent variable may be different (i.e., $r_i \neq r_j$). We present the minimal state space criteria by which the state space of the latent variable in the conditional set is identifiable.

\begin{theorem}[Minimal state space criteria]
In the discrete LSM model, suppose Assumption 2.2 $\sim$ 2.4 holds. For any two observed variables $X_i$ and $X_j$, let $L_i$ and $L_j$ be their latent parent respectively, and let $r_1$ be the dimension of $\mathrm{supp}(L_i)$ and $r_2$ be the dimension of $\mathrm{supp}(L_j)$, we have $\mathrm{Rank}(\mathcal{T}_{(X_1,X_2)}) = min(r_1, r_2)$.

\end{theorem}

\begin{proof}
 In the discrete LSM model, for two observed variables, we have $r < x_i$ and any one of the latent parents can d-separate $X_i$ from $X_i$. Based on the graphical criteria of the tensor rank condition, one can see that the rank is the latent parent with the minimal dimension of support.
\end{proof}

Based on the minimal state space criteria, one can directly extend the identification of the discrete LSM model to the setting where the latent state space can be different.

\subsection{Identification of causal cluster}

\begin{proposition}[Identification of causal cluster in different state space]

In the discrete LSM mode, suppose Assumption 2.2 $\sim$ Assumption 2.4 holds. For three disjoint observed variables $X_i, X_j, X_k \in \mathbf{X}$, then $\{X_i, X_j, X_k\}$ share the same latent parent  if $\mathrm{Rank}(\mathcal{T}_{(X_i, X_j, X_k, X_s)}) = r$ for any $X_s \in \mathbf{X} \setminus \{X_i, X_j, X_k\}$, where $r = \mathrm{Rank}(\mathcal{T}_{(X_i, X_j)}) = \mathrm{Rank}(\mathcal{T}_{(X_i, X_k)}) = \mathrm{Rank}(\mathcal{T}_{(X_k, X_j)})$.

\end{proposition}

\begin{proof}
    If $\mathrm{Rank}(\mathcal{T}_{(X_i, X_j, X_k, X_s)}) = r$, there exist a variable set $\mathbf{S}$ with $|\mathrm{supp}(\mathbf{S})| = r$ that d-separates any pair variables in $\{X_i, X_j, X_k, X_s\}$, according to the graphical criteria of tensor rank condition. In the discret LSM model and $r < d_i$ (any support dimension of latent variable is less than the support dimension of observed variable), $r = \mathrm{Rank}(\mathcal{T}_{(X_i, X_j)}) = \mathrm{Rank}(\mathcal{T}_{(X_i, X_k)}) = \mathrm{Rank}(\mathcal{T}_{(X_k, X_j)})$, we have for any $X_i, X_j \in \{X_i, X_j, X_k, X_s\}$, the conditional set is one of the latent parent of $X_i$ and $X_j$. If $X_i, X_j$ and $X_k$ do not share the common latent parent, without loss of generality, let $L_1$ be the parent of $X_i, X_j$ and $L_2$ be the parent of $X_k$, there exist the $X_s \in Ch_{L_2}$ such that $X_k$ and $X_s$ cannot be d-separated given $L_1$ or $X_i$ and $X_j$ cannot be d-separated given $L_2$. By the graphical criteria of tensor rank condition, $\mathrm{Rank}(\mathcal{T}_{(X_i, X_j, X_k, X_s)}) \neq |\mathrm{supp}(L_1)|$, or  $\mathrm{Rank}(\mathcal{T}_{(X_i, X_j, X_k, X_s)}) \neq |\mathrm{supp}(L_2)|$. Let $r = \mathrm{min}(|\mathrm{supp}(L_1)|, |\mathrm{supp}(L_2)|)$, we have $\mathrm{Rank}(\mathcal{T}_{(X_i, X_j, X_k, X_s)}) \neq r$ if $X_i, X_j$ and $X_j$ are not a causal cluster.
\end{proof}

One can properly adjust the search algorithm such that the causal cluster can be identified, by the minimal state space criteria. The algorithm is presented as follows (Algorithm \ref{alg:algorithm_clu2}).

\begin{algorithm}[htp]
    \caption{Identifying the causal cluster (different latent space)}
    \label{alg:algorithm_clu2}
    \textbf{Input}: Data from a set of measured variables $\mathbf{X}_{\mathcal{G}}$\\
    %\textbf{Parameter}: Optional list of parameters\\
    \textbf{Output}: Causal cluster $\mathcal{C}$
    \begin{algorithmic}[1] %[1] enables line numbers
        \STATE Initialize the causal cluster set $\mathcal{C} \coloneqq \emptyset$, and $\mathcal{G}^\prime= \emptyset$;
        \STATE \textbf{Begin} the recursive procedure 
        \REPEAT{
        \FOR{each $X_i, X_j$ and $X_k$ $\in \mathbf{X}$}
        \STATE \Alcomment{\textit{// Apply the minimal state space criteria}}
        \STATE r = min($\mathrm{Rank}(\mathcal{T}_{\{X_i, X_j\}})$,$\mathrm{Rank}(\mathcal{T}_{\{X_i, X_k\}})$,$\mathrm{Rank}(\mathcal{T}_{\{X_k, X_j\}})$ );
        
        \IF{$\mathrm{Rank}(\mathcal{T}_{\{X_i, X_j, X_k, X_s\}})=r$, for all $X_s \in \mathbf{X} \setminus \{X_i, X_j, X_k\}$}
        \STATE $\mathbf{C} = \mathbf{C} \cup \{\{X_i, X_j, X_k\}\}$;
        \ENDIF
        \ENDFOR
        }
        \UNTIL{no causal cluster is found.}
        \STATE \Alcomment{\textit{// Merging cluster and introducing latent variables}}
        \STATE Merge all the overlapping sets in $\mathbf{C}$ by Prop. \ref{pro_merger}.
        \FOR{each $C_i \in \mathbf{C}$}
        \STATE Introduce a latent variable$L_i$ for $C_i$;
        \STATE $\mathcal{G} = \mathcal{G} \cup \{L_i \to X_j|X_j \in C_i \}$.
        \ENDFOR
        \STATE \textbf{return} Graph $\mathcal{G}$ and causal cluster $\mathcal{C}$.
    \end{algorithmic}
\end{algorithm}

\subsection{Conditional independence test among latent variables}

\begin{proposition}[conditional independence among latent variables in different state space]

In the discrete LSM model, suppose Assumption 2.2 $\sim$
Assumption 2.4 holds. Let $X_i$ and $X_j$ be the pure child of $L_i$ and $L_j$ respectively, $\mathbf{X}_{p1}$ and $\mathbf{X}_{p2}$ be two disjoint child set of the latent set $\mathbf{L}_p$ with $|\mathbf{X}_{p1}| = |\mathbf{X}_{p2}| = |\mathbf{L}_p|$, then $L_i \bot L_j | \mathbf{L}_p$ if and only if  $\mathrm{Rank}(\mathcal{T}_{(X_i, X_j, \mathbf{X}_{p1}, \mathbf{X}_{p2})}) = r$, where $r = \prod_{L_i \in \mathbf{L}_p} |\mathrm{supp}(L_i)|$.

\end{proposition}

\begin{proof}
'If' part: in the discrete LSM, we have $r_i < d_j$ for any $i\in [k], j\in [p]$, where $k$ is the number of latent variables while $p$ is the number of observed variables. In the causal graph, for $X_{q1} \in \mathbf{X}_{p1}$ and $X_{q2} \in \mathbf{X}_{p2}$, $X_{q1}, X_{q2} \in Ch(L_t)$ for $\forall L_t \in \mathbf{L}_p$, we have $L_t$ is the only conditional set that d-separates $X_{q1}$ and $X_{q2}$ with minimal support dimension. Thus, $\mathbf{L}_p$ also be the minimal conditional set that d-separates any pair variables in $\mathbf{X}_{p1} \cup \mathbf{X}_{p2}$. Now, if $X_i$ and $X_j$ are d-separated by $\mathbf{L}_p$, according to the graphical criteria of tensor rank condition, one have $\mathrm{Rank}(\mathcal{T}_{(X_i, X_j, \mathbf{X}_{p1}, \mathbf{X}_{p2})}) = |\mathrm{supp}(\mathbf{L}_p)|$. Since $\mathbf{L}_p$ is the joint distribution of latent variable set, we have $r = \prod_{L_i \in \mathbf{L}_p} |\mathrm{supp}(L_i)|$. 

'Only if' part: on the other hand, if $X_i$ and $X_j$ are not d-separated by $\mathbf{L}_p$, for example, $L_i \to L_j$ in the causal graph, then $\mathbf{L}_p$ is not a conditional set that d-separates all pair variables in $\{X_i, X_j, \mathbf{X}_{p1}, \mathbf{X}_{p2}\}$. According to the graphical criteria of tensor rank condition, we have $\mathrm{Rank}(\mathcal{T}_{(X_i, X_j, \mathbf{X}_{p1}, \mathbf{X}_{p2})}) \neq |\mathrm{supp}(\mathbf{L}_p)|$, i.e., $\mathrm{Rank}(\mathcal{T}_{(X_i, X_j, \mathbf{X}_{p1}, \mathbf{X}_{p2})}) \neq \prod_{L_i \in \mathbf{L}_p} |\mathrm{supp}(L_i)|$. This completes the proof.
\end{proof}

In particular, $|\mathrm{supp}(L_i)|$ can be identified by their pure child variable, according to minimal state space criteria. An intuition illustration is by mapping the conditional set variable $\mathbf{L}_p$ to one new latent variable $\tilde{L}$ (i.e., vectorization), the graphical criteria of causal cluster still hold, e.g., $\{\mathbf{X}_{p1}, \mathbf{X}_{p2}, X_i\}$ is a causal cluster that shares a common parent $\tilde{L}$. However, such a map will exponentially increase the dimension of the latent variable support. One issue will be raised: the observed variable may have a smaller support dimension than the latent variables such that the d-separation relations among latent variables cannot be examined. Thus, it is necessary to study when and how the testability of d-separation holds. The result is provided in Prop. \ref{pro_test}.

\begin{remark}[Testability of d-separation]\label{pro_test}
     For an n-way tensor $\mathcal{T}_{\{X_1,..., X_n\}}$ that is used to test the d-separation relations among latent variables, such a CI relation is testable if $\prod^{|\mathbf{L}_p|}_{j=1}r_j^{|\mathbf{L}_p|} < \prod^{n}_{i=1}d_i - \text{max}(d_1,...,d_n)$.

\end{remark}

\begin{proof}
    We prove it by contradiction. If $\prod^{|\mathbf{L}_p|}_{j=1}r_j^{|\mathbf{L}_p|} > \prod^{n}_{i=1}d_i - \text{max}(d_1,...,d_n)$, assume that $d_n = max (d_1, \cdots, d_n)$ where $d_i$ is the support of observed $X_i$, there are 
    \begin{equation}
        \begin{aligned}
            & \mathcal{T}_{(X_1 \cdots X_n)}
            = \sum_{\mathbb{P}(X_1 \cdots X_{n-1})}  \mathbb{P}(X_1|X_1  \cdots X_{n-1}) \cdots \mathbb{P}(X_n|X_1  \cdots X_{n-1}) \mathbb{P}(X_1  \cdots X_{n-1})\\
            & = \sum^{\prod^{n-1}_{j=1}d_j}_{i = 1} \mathbf{u}^{(i)}_1 \otimes \cdots \otimes \mathbf{u}^{(i)}_1,
        \end{aligned}
    \end{equation}

    which is a smaller rank-one decomposition than $\mathbf{L}_p$ with support $r^{|\mathbf{L}_p|}$. According to the definition of tensor rank and the graphical criteria of tensor rank, we have $\mathrm{Rank}(\mathcal{T}_{(X_1 \cdots X_n)}) = \prod^{n}_{i=1}d_i - \text{max}(d_1,...,d_n)$. That is, no matter whether the conditional independent relations hold given $\mathbf{L}_p$, the rank of tensor still be the dimension of the support of $\mathbb{P}(X_1 \cdots X_{n-1})$. It means that the CI relations can not be detected.

\end{proof}

In other words, if the support of the conditional set is more than the dimension of observed variables, then the minimal rank-one decomposition of the joint distribution will lead to $\mathbb{P}(X_1, \cdots, X_n) = \sum_{\{X_1, \cdots, X_{n-1}\}} \mathbb{P}(X_1, \cdots, X_{n}|X_1, ..., X_{n-1})P(X_1, \cdots, X_{n-1})$. Thus, if the increasing of latent state space is less than the sum of tensor dimensions, the CI relations among latent variable are testable. Due to assuming that $d > r$, the CI relations among latent variables are generally testable when the causal structure is sparse.

\section{Discussion with the Hierarchical Structures}

Actually, our result can be extended to a specific hierarchical structure, by constraining the structure of hidden variables. For instance, consider a hierarchical structure in which each latent variable is required to have at least three pure children (whether latent or observed) and one additional neighboring variable. An illustration of this type of structure is provided in Fig. \ref{fig:hierarchical}. Assume that all latent variables have the same dimension of support, and that this dimension is smaller than that of the observed variables. Under these conditions, causal clusters at the bottom level can still be identified, as demonstrated by Proposition 2. For instance, the sets $\{X_1, X_2, X_3\}$, $\{X_4, X_5, X_6\}$, $\{X_7, X_8, X_9\}$, and $\{X_{10}, X_{11}, X_{12}\}$ are recognized as four distinct causal clusters. The pure measured variables from each cluster can act as surrogates for their corresponding latent parents, allowing the causal cluster learning procedure to be repeated. For example, if $X_1$ serves as the surrogate for $L_2$, and ${X_4, X_7, X_{10}}$ as surrogates for ${L_3, L_4, L_5}$, then $\{L_2, L_3, L_4, L_5\}$ can be identified as a cluster according to the graphical criteria of tensor rank. Thus, the specific hierarchical structure is identifiable by designing the proper search algorithm making use of the tensor rank condition. We will explore these results in future works.

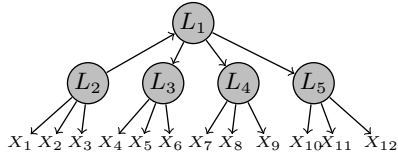
\begin{figure}[htp]
    \centering
    		\begin{tikzpicture}[scale=1.0, line width=0.5pt, inner sep=0.2mm, shorten >=.1pt, shorten <=.1pt]
		\draw (2.4, 2.4) node(L0) [circle,inner sep=0pt, fill=gray!50,draw] {{\footnotesize\,$L_1$\,}};

		%top level
		\draw (1.0, 1.6) node(L1) [circle, inner sep=0pt, fill=gray!50,draw] {{\footnotesize\,$L_2$\,}};
		\draw (2.0, 1.6) node(L2) [circle,inner sep=0pt,  fill=gray!50,draw] {{\footnotesize\,$L_3$\,}};
% 		\draw (2.4, 1.6) node(L3) [circle,inner sep=0pt, minimum size=0.2cm, fill=gray!60,draw] {};
		\draw (3.0, 1.6) node(L3) [circle,inner sep=0pt,fill=gray!50,draw] {{\footnotesize\,$L_4$\,}};
		\draw (4.0, 1.6) node(L4) [circle,inner sep=0pt,fill=gray!50,draw] {{\footnotesize\,$L_5$\,}};
% 		\draw (2.8, 1.6) node(L4) [circle, fill=gray!60,draw] {};

		%
		\draw[->] (L1) -- (L0) node[pos=0.5,sloped,above] {};
		\draw[->] (L0) -- (L2) node[pos=0.5,sloped,above] {};
		\draw[->] (L0) -- (L3) node[pos=0.5,sloped,above] {};
		\draw[->] (L0) -- (L4) node[pos=0.5,sloped,above] {};

		%third level
		\draw (0.1, 0.8) node(X1) [] {{\tiny\,$X_{1}$\,}};
		\draw (0.5, 0.8) node(X2) [] {{\tiny\,$X_2$\,}};
		\draw (0.9, 0.8) node(X3) [] {{\tiny\,$X_3$\,}};
		\draw (1.3, 0.8) node(X4) [] {{\tiny\,$X_4$\,}};
		\draw (1.7, 0.8) node(X5) [] {{\tiny\,$X_5$\,}};
		\draw (2.1,0.8) node(X6) [] {{\tiny\,$X_{6}$\,}};
		\draw (2.5,0.8) node(X7) [] {{\tiny\,$X_{7}$\,}};
		\draw (2.9,0.8) node(X8) [] {{\tiny\,$X_{8}$\,}};
		\draw (3.4,0.8) node(X9) [] {{\tiny\,$X_{9}$\,}};
		\draw (3.9,0.8) node(X10) [] {{\tiny\,$X_{10}$\,}};
		\draw (4.3,0.8) node(X11) [] {{\tiny\,$X_{11}$\,}};
		\draw (4.9,0.8) node(X12) [] {{\tiny\,$X_{12}$\,}};
		\draw[->] (L1) -- (X1) node[pos=0.5,sloped,above] {};
		\draw[->] (L1) -- (X2) node[pos=0.5,sloped,above] {};
		\draw[->] (L1) -- (X3) node[pos=0.5,sloped,above] {};
		\draw[->] (L2) -- (X4) node[pos=0.5,sloped,above] {};
		\draw[->] (L2) -- (X5) node[pos=0.5,sloped,above] {};
		\draw[->] (L2) -- (X6) node[pos=0.5,sloped,above] {};
		\draw[->] (L3) -- (X7) node[pos=0.5,sloped,above] {};
		\draw[->] (L3) -- (X8) node[pos=0.5,sloped,above] {};
		\draw[->] (L3) -- (X9) node[pos=0.5,sloped,above] {};
		\draw[->] (L4) -- (X10) node[pos=0.5,sloped,above] {};
		\draw[->] (L4) -- (X11) node[pos=0.5,sloped,above] {};
		\draw[->] (L4) -- (X12) node[pos=0.5,sloped,above] {};
	\end{tikzpicture}~~~
    \caption{Example of hierarchical structure.}
    \label{fig:hierarchical}
\end{figure}

\section{Practical Estimation of Tensor Rank}

Here, we describe the practical implementation of tensor rank estimation. To alleviate the problem of local optima during tensor decomposition, we initiate the process from multiple random starting points. We then perform the tensor decomposition from each of these points and subsequently select the decomposition that yields the smallest reconstruction error as our final result. The procedure is summarised in the Algorithm \ref{alg:algorithm_pra}.

 \begin{algorithm}[h]
    \caption{Practial tensor rank estimation}
    \label{alg:algorithm_pra}
    \textbf{Input}: $ \mathbf{X}_p = \{X_i, X_j,  ..., X_n\}$, iteration number $n$, threshold $\varepsilon_r$ and tested rank $r$\\
    %\textbf{Parameter}: Optional list of parameters\\
    \textbf{Output}: Boolean of rank test
    \begin{algorithmic}[1] %[1] enables line numbers
    	\STATE Initialize the minimal reconstructed error $E_{min} = +\infty$;
        \FOR{$i < n$ or $E_{min} \leq \varepsilon$}
        \STATE $i\leftarrow i+1$;
        \STATE $\tilde{\mathcal{T}} \leftarrow$ \text{non-negative-parafac}$(\mathcal{T}_{(\mathbf{X}_p)}, r)$;
        \STATE E $\leftarrow  \Vert \mathcal{T}_{(\mathbf{X}_p)}-$  $\tilde{\mathcal{T}}$ $\Vert$;
        \IF{$E \leq E_{min}$}
        \STATE $E_{min} = E$;
        \ENDIF
        \ENDFOR
        \STATE p-val $\leftarrow$ Chi-Square($\mathcal{X}^2$)-Test(vec($\mathcal{T}$), vec($\tilde{\mathcal{T}}$));
        \IF{p-val < $\varepsilon$}
        \STATE \textbf{return True};
        \ENDIF
        \STATE \textbf{return False}.
    \end{algorithmic}
\end{algorithm}

Beside, in the PC-TENSKR-RANK algorithm, to further identify the V-structure among latent variables, the statistic independent test among latent variables is required, which can be tested by following.

\begin{remark}[Statistic independent between latent variables]
Give the measured variable $X_i$ and $X_j$ of latent variable $L_i$ and $L_j$, then $L_i \Vbar L_j$ if $\text{Rank}(\mathbb{P}(X_i, X_j)) = 1$.
\end{remark}

\begin{proof}
    Since $L_i \Vbar L_j$, we have $X_i \Vbar X_j$ also hold in the causal graph. We have $\mathbb{P}(X_i X_j) = \mathbb{P}(X_i) \mathbb{P}(X_j) = \mathcal{T}_{(X_i)} \otimes \mathcal{T}_{(X_j)}$. According to the definition of tensor rank, $\text{Rank}(\mathcal{T}_{(X_i, X_j)}) = 1$.
\end{proof}

\subsection{Goodness of fit test for CI test among latent variables}

Although the proposed tool is theoretically testable, it still is an approximate estimation of tensor rank by heuristic-based CP decomposition in practice. How to consider a more robust approach to examine the tensor rank still be an open problem in the related literature. It significantly restricts the application scope and performance of our structure learning algorithm. However, we want to emphasize that the main contribution of our work is building the graphical criteria of tensor rank and using it to answer the identification of causal structure in a discrete LVM model. To the best of our knowledge, this is the first algorithm that can identify the causal structure of discrete latent variables without structural constraints, including the measurement model and structure model. 

Next, we will show that how the CI relations among latent variables can be distinguished by testing the goodness of fit test. Consider a four latent variables structure as shown in Fig. \ref{fig:CItest}, a chain structure among four latent variables in which each latent variable has two pure observed variables. The data generation process follows the discret LSM model (see the description in the simulation studies section) and the sample size is $50k$. We check the CI relations between any $L_i, L_j \in \{L_1, L_2, L_3, L_4\}$ given $L_p \in \{L_1, L_2, L_3, L_4\} \setminus \{L_i, L_j\}$. The results are reported in the right side of \ref{fig:CItest}, in which each red point represents a CI test result, e.g., the second point in the graph represent to test $L_2 \Vbar L_4|L_1$ by examining $\mathrm{Rank}(\mathcal{T}_{(X_3, X_7, X_1, X_2)}) =2$. One can see that the p-value returned by the goodness of fit test is lower than 0.05, which means that we will tend to reject the null hypothesis, i.e., $\mathrm{Rank}(\mathcal{T}_{(X_3, X_7, X_1, X_2)}) \neq 2$. By sorting all CI test results, one can see that the true CI relations can be identified by setting the significant level to be 0.05.

\begin{figure}[htp]
    \centering
    \includegraphics[width=0.35\textwidth]{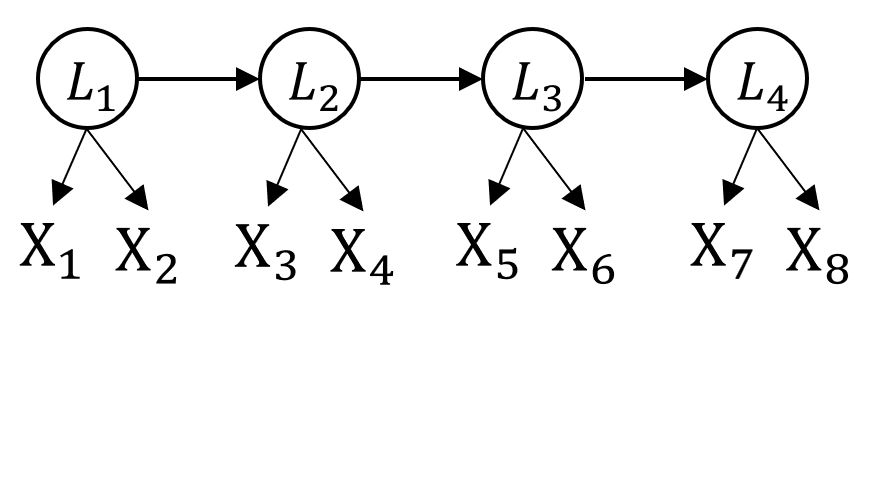}~~~~~~~~
    \includegraphics[width=0.45\textwidth]{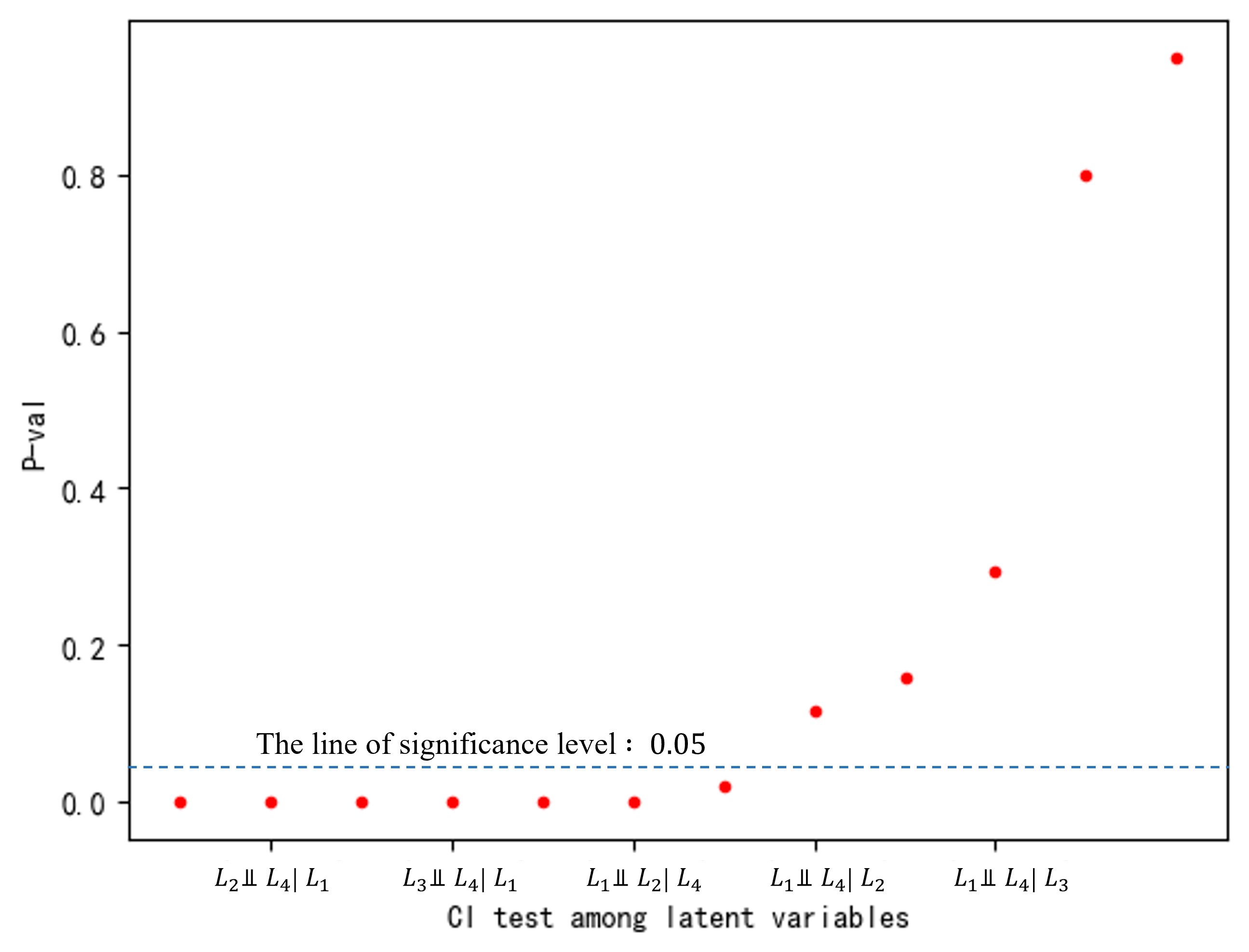}
    \caption{Goodness of fit test for conditional independent test among latent variables}
    \label{fig:CItest}
\end{figure}

% \begin{figure}[t]
% \begin{minipage}{0.5\textwidth}
% \textbf{Illustrating for experimental result}: Consider the sub-structure $L_1 \to L_2$ in Fig. \ref{fig:illust_simple_example} (a). We test all four-way tensor combinations for learning causal cluters according to Prop. \ref{prop_cluster}. The data generation process follows the discrete LVM model and the assumption set, where the $r=2$ and $d=3$ for all observed variables. The sample size is set to 50k. The result is shown the right side. One can see that for the causal cluster, the p-value of rank test is relative higher than 0.005, which means that the causal cluster can be identified by accepting the rank test to be two, setting the threshold to 0.005.\\
% \end{minipage}
% \begin{minipage}{0.5\textwidth}
% \includegraphics[width=0.95\textwidth]{fig/PVAL.png}
% \label{fig:t-test}
% \end{minipage}
% \vspace{-0.2cm}
%     \caption{Illustrating the significant difference between p-value of rank test in learning causal clusters for two latent factor structures ($L_1 \to L_2$ and $\{X_1, X_2, X_3\}$, $\{Y_1, Y_2, Y_3\}$ are their measured variable set, respectively), where the red point implies the reconstructed error of different rank decomposition.}
%     \label{fig:t-test}
% \end{figure}

\section{More Experimental Results}

In this section, we provide the information required to reproduce our results reported in the main text. We further conduct additional simulation experiments to validate the efficiency of the proposed algorithm (Appendix \ref{sec_dif}).

We first give a details definition of evaluation metrics. Specifically, the performance of causal cluster is evaluated by following scores for the output model $G_{out}$ from each algorithm, where the true graph is labelled $G$:

\begin{itemize}
    \item \textbf{latent omission}, the number of latents in $G$ that do not appear in $G_{out}$ divided by the total number of true latents in $G$;
    \item \textbf{latent commission}, the number of latents in $G_{out}$ that could not be mapped to a latent in $G$ divided by the total number of true latents in $G$;
    \item \textbf{mismeasurement}, the number of observed variables in $G_{out}$ that are measuring at least one wrong latent divided by the number of observed variables in $G$;
\end{itemize}

Moreover, we use the following metric to evaluate the performance of causal structure among latent variables:

\begin{itemize}
    \item \textbf{edge omission (EO)}, the number of edges in the structural model of $G$ that do not appear in $G_{out}$ divided by the possible number of edge omissions;
    \item \textbf{edge commission (EC)}, the number of edges in the structural model of $G_{out}$ that do not exist in $G$ divided by the possible number of edge commissions;
    \item \textbf{orientation omission (OO)}, the number of arrows in the structural model of $G$ that do not appear in $G_{out}$ divided by the possible number of orientation omissions in $G$;
\end{itemize}

These evaluation indicators are derived from \cite{Silva-linearlvModel}.

Next, we give a concrete implementation of baseline methods.

\textbf{BayPy}: The Bayesian Pyramid Mode (BayPy) is a discrete latent variable structure learning method that assumes the latent structure is a pyramid structure and the latent variable is binary. We use the implementation of \cite{gu2023bayesian}. We set the iteration parameter to 1500 and set the search upper bound of the number of latent variables to 5.

\textbf{LTM}: The latent tree model, is a classic method for learning gaussian or binary latent tree structure. We use the implementation from \cite{choi2011latenttree}. Specifically, we use the Recursive Grouping (RG) Algorithm in \cite{choi2011latenttree} (since it has better performance), and use the discrete information distance to learn the structure of the discrete LSM model.

\textbf{BPC}: The Building Pure Cluster (BPC) algorithm \cite{Silva-linearlvModel} is a classic causal discovery method for the linear latent variable model. We use the implementation from the Tetrad Project package \footnote{https://github.com/cmu-phil/tetrad}.

\textbf{Non-negative CP decomposition}: To perform non-negative CP decomposition in our algorithm, we use the implementation from the python package, \textit{tensorly} \footnote{https://github.com/tensorly/tensorly}, and set the maximum iteration parameter to 1000, the cvg criterion parameter to "rec\_error".

\textbf{Data Generation}: To generate the probability contingency table or conditional probability contingency table for latent variables and observed variables in our simulation studies, we use the implementation from the python package, \textit{pgmpy} \footnote{https://github.com/pgmpy/pgmpy}. The package provides the function to generate observed data according to the probability contingency table.

Finally, we aim to demonstrate the correctness of our methods in handling cases involving latent variables with varying state spaces. Specifically, we consider the structure model with star structure and the $SM_3$ structure, and the measurement model with $MM_1$. The data generation process follows the description in the main context and we let the support of latent variable $L_1, L_2$ to be $\{0, 1\}$ and the support of latent variable $L_3, L_4$ to be $\{0, 1, 2\}$, for simulating the case that latent variable has different latent support. Besides, the support of the observed variable is $\{0,1,2,3\}$. In our implementation, the significant levels for testing the rank of the matrix and tensor are set to 0.005 and 0.05, respectively. The coefficient of probability contingency tables is generated randomly, ranging from $[0.1,0.8]$, constraining the sum of them to be one.

The results are presented in Table \ref{Tab_3} and Table \ref{tab_4}. The performance of our method appears superior in scenarios where latent variables have differing state spaces. This improvement is attributed to the reduction in the frequency of higher-order tensor rank testing, facilitated by evaluating the consistency of ranks such as $\mathrm{Rank}(\mathcal{T}{(X_i, X_j)}) = \mathrm{Rank}(\mathcal{T}{(X_i, X_k)}) = \mathrm{Rank}(\mathcal{T}_{(X_k, X_j)})$. In contrast, the BayPy approach underperforms in latent structure learning, likely due to its assumption of a pyramid structure with a top-down directionality and no peer-level connections among latent variables. Additionally, the Latent Tree Model (LTM) shows weaker performance in cluster learning, possibly because it was originally designed to handle only binary discrete variables.

\begin{center}
\begin{table*}[htp!]
% \centering\color{red}
%  \vspace{-3mm}
% \setlength{\abovecaptionskip}{1pt}
% 	\setlength{\belowcaptionskip}{1pt}
	\small
	\center \caption{Results on learning pure measurement models in the case that latent variables have different state spaces. Lower value means higher accuracy.}
% 	\vspace{-3mm}
	\label{tab:mixed distribution}
	\resizebox{1.0\textwidth}{!}{
	\begin{tabular}{cccccccccccccc}
		\hline  \multicolumn{2}{c}{} &\multicolumn{4}{c}{\textbf{Latent omission}} & \multicolumn{4}{c}{\textbf{Latent commission}} & \multicolumn{4}{c}{\textbf{Mismeasurements}}\\
		%\hline 
		\multicolumn{2}{c}{Algorithm} & \textbf{Our} & BayPy  & LTM & BPC & \textbf{Our} & BayPy & LTM & BPC & \textbf{Our} & BayPy & LTM & BPC \\
		\hline 
		 & 5k & 0.09(3) & 0.20(4) & 0.23(5) & 0.96(10) 
        		 & 0.00(0) & 0.20(4)& 0.00(0) & 0.00(0) 
        		 & 0.03(1) & 0.18(4) & 0.23(5) & 0.00(0) \\
		% \cline{2-14}
		{\emph{$Star + MM_1$}} &10k &  0.06(2) &0.17(3) & 0.13(4) & 0.96(10) 
                		& 0.00(0)&0.15(3) & 0.00(0) & 0.00(0) 
                		& 0.00(0) & 0.15(3)& 0.13(4) & 0.00(0) \\
		% \cline{2-14}
		&50k & 0.00(0) & 0.13(3) & 0.10(3) & 0.93(10) 
        		& 0.00(0) & 0.15(3)& 0.00(0) & 0.00(0) 
        		& 0.00(0) & 0.13(3) & 0.10(3) & 0.00(0) \\
		\hline 
		 %case 2
		 & 5k 
		 & 0.12(3) & 0.33(7) & 0.55(10) & 0.96(10) 
		 & 0.00(0) & 0.30(6) & 0.00(0) & 0.00(0)
		 & 0.06(2) & 0.30(7) & 0.55(10) & 0.00(0)  \\
		% \cline{2-14}
		{\emph{$SM_3 + MM_1$}} &10k 
		& 0.06(2) & 0.26(5) & 0.50(10) & 0.93(10) 
		& 0.00(0) & 0.20(5) & 0.00(0) & 0.00(0) 
		& 0.00(0) & 0.19(5) & 0.50(10) & 0.00(0)\\
		% \cline{2-14} 
            &50k
		& 0.03(1) & 0.20(4) & 0.50(10) & 0.93(10)
		& 0.00(0) & 0.15(4) & 0.00(0) & 0.00(0)
		& 0.00(0) & 0.11(4) & 0.50(10)  & 0.00(0)\\
		\hline 
	\end{tabular}}
% 	\vspace{-1mm}
\label{Tab_3}
\end{table*}
\end{center}

\begin{center}
\begin{table*}[htp!]
% \centering\color{red}
%  \vspace{-3mm}
% \setlength{\abovecaptionskip}{1pt}
% 	\setlength{\belowcaptionskip}{1pt}
	\small
	\center \caption{Results on learning the structure model in the case that latent variables have different state spaces. The symbol '-' indicates that the current method does not output this information. Lower value means higher accuracy.}
% 	\vspace{-3mm}
	\label{tab:mixed distribution}
	\resizebox{1.0\textwidth}{!}{
	\begin{tabular}{cccccccccccccc}
		\hline  \multicolumn{2}{c}{} &\multicolumn{4}{c}{\textbf{Edge omission}} & \multicolumn{4}{c}{\textbf{Edge commission}} & \multicolumn{4}{c}{\textbf{Orientation omission}}\\
		%\hline 
		\multicolumn{2}{c}{Algorithm} & \textbf{Our} & BayPy  & LTM & BPC & \textbf{Our} & BayPy & LTM & BPC & \textbf{Our} & BayPy & LTM & BPC \\
		\hline 
		 & 5k & 0.00(0) & 1.00(10) & 0.26(8) & 1.00(10) 
        		 & 0.10(1) & 0.00(0)& 0.00(0) & 0.00(0) 
        		 & 0.10(1) & 1.00(10) & -- & 1.00(0) \\
		% \cline{2-14}
		{\emph{Star+$MM_1$}} &10k &  0.00(0) &1.00(10) & 0.23(6) & 1.00(10) 
                		& 0.00(0) &0.02(1) & 0.0(0) & 0.00(0) 
                		& 0.00(0) & 1.00(10) & -- & 1.00(0) \\
		% \cline{2-14}
		&50k & 0.00(0) & 1.00(10) & 0.10(3) & 1.00(10) 
        		& 0.00(0) & 0.00(0)& 0.00(0) & 0.00(0)
        		& 0.00(0) & 1.00(10) & -- & 1.00(0) \\
		\hline 
		 %case 2
		 & 5k 
		 & 0.15(3) & 1.00(10) & 0.16(6) & 1.00(10) 
		 & 0.10(1) & 0.00(0)& 0.00(0) & 0.00(0)
		 & 0.00(0) & 0.00(0) & -- & 0.00(0)  \\
		% \cline{2-14}
		{\emph{$SM_3 + MM_1$}} &10k 
		& 0.05(1) & 1.00(10) & 0.13(4) & 1.00(10) 
		& 0.01(1) & 0.00(0)& 0.00(0) & 0.00(0) 
		& 0.00(0) & 0.00(0) & -- & 0.00(0) \\
		% \cline{2-14} 
            &50k
		& 0.00(0) & 1.00(10) & 0.10(3) & 1.00(10)
		& 0.00(0) & 0.00(0)& 0.00(0) & 0.00(0)
		& 0.00(0) & 0.00(0) & -- & 0.00(0) \\
		\hline 
	\end{tabular}}
% 	\vspace{-1mm}
\label{tab_4}
\end{table*}
\end{center}

\section{More Details of Real-world Dataset}

For the political efficacy data \cite{aish1990panel}, by identifying the support of latent variable to be two, one can identify the causal cluster $\{\text{'NOCARE'},\text{'TOUCH'}, \text{'INTEREST'}\}$, $\{\text{'NOSAY'},\text{'VOTING'}, \text{'COMPLEX'}\}$. These clusters correspond to the latent variables 'EFFICACY' and 'RESPONSE', respectively. In our implementation, we set the significance level to 0.0015. The result is aligned with the ground truth provided in \cite{joreskog1996lisrel}.

% Chow (2000) used five indicators of self, esteem (Selfest) and four indicators of depression (Depress) to formulate a measurement model for the latent variables Self Esteem and Depression. 

% The Depression dataset \cite{einstein2000relationship,lovibond1995structure} is a small part of a real data file that includes each woman's scores on the Edinburgh Postnatal Depression Scale and the Depression, Anxiety, and Stress scales. 

For the depress dataset, the ground truth structure \cite{joreskog1996lisrel} includes three latent factors: Self-esteem, Depression, and  Impulsiveness, with the corresponding observed clusters:
\begin{itemize}
    \item $\{\text{'SELF1'},\text{'SELF2'}, \text{'SELF3'}, \text{'SELF4'}, \text{'SELF5'}\}$;
    \item  $\{\text{'DEPRES1'}, \text{'DEPRES2'},\text{'DEPRES3'}, \text{'DEPRES4'}\}$;
    \item $\{\text{'IMPULS1'}, \text{'IMPULS2'}, \text{'IMPULS3'}\}$.
\end{itemize}

In our implementation, we utilize a bootstrapping resampling approach to enhance the statistical properties of the data. Following the extended results outlined in Appendix \ref{sec_dif}, we first identify the dimension of support for the factors Self-esteem and Depression as four, and set the dimension of support for Impulsiveness at three. The significance level is set to 0.025. The results of our algorithm are presented as follows.

\begin{itemize}
    \item $\{\text{'SELF1'},\text{'SELF2'}, \text{'SELF3'}, \text{'SELF5'}\}$;
    \item  $\{\text{'DEPRES1'},\text{'DEPRES3'}, \text{'DEPRES4'}\}$;
    \item $\{\text{'IMPULS1'}, \text{'IMPULS2'}, \text{'IMPULS3'}\}$.
\end{itemize}

One can see that our algorithm can learn three causal clusters corresponding to three latent factors. Such a result shows that our method finds all latent variables from the depress data. In the latent structure learning process, the PC-TENSOR-RANK algorithm outputs the fully connected graph of the three latent factors, indicating the absence of conditional independence (CI) relations between them. One possible reason is there may be more potential factor interaction structure \cite{salles2024indirect}.

\end{document}